\documentclass[final,links,onefignum,onetabnum]{siamart220329}
\usepackage{amsmath, amssymb, amsfonts}
\usepackage{threeparttable}
\usepackage{graphicx,epstopdf}
\usepackage{subfigure}
\usepackage{comment}

\usepackage{algorithm}
\usepackage{algorithmic}
\usepackage{mathrsfs} 
\usepackage{multirow} 
\usepackage{multicol} 
\usepackage{booktabs}
\usepackage{threeparttable}
\usepackage{upgreek}

\newsiamremark{remark}{Remark}

\AtBeginDocument{%
  \paperwidth=\dimexpr
    1in + \oddsidemargin
    + \textwidth
    + 1in + \oddsidemargin
  \relax
  \paperheight=\dimexpr
    1in + \topmargin
    + \headheight + \headsep
    + \textheight
    + 1in + \topmargin
  \relax
  \usepackage[pass]{geometry}\relax
}

\ifpdf
\hypersetup{
  pdftitle={Stochastic diagonal estimation with adaptive parameter selection},
  pdfauthor={Zongyuan Han, Wenhao Li, Shengxin Zhu}
}
\fi

\headers{Stochastic diagonal estimation with adaptive parameter selection}{Zongyuan Han, Wenhao Li and Shengxin Zhu}

\title{Stochastic diagonal estimation with adaptive parameter selection\thanks{version \today.
\funding{This work was funded by the Natural Science Foundation of China (12271047); Guangdong Provincial Key Laboratory of Interdisciplinary Research and Application for Data Science, BNU-HKBU United International College (2022B1212010006); UIC research grant(UICR0400008-21; UICR04202405-21); Guangdong College Enhancement and Innovation Program(2021ZDZX1046).}}}

\author{
Zongyuan HAN\thanks{Laboratory of Mathematics and Complex Systems, Ministry of Education;  School of Mathematical Sciences, Beijing Normal University, Beijing 100875, P.R. China.}
\and Wenhao Li\thanks{Hong Kong Baptist University, Hong Kong, P.R. China;  Guangdong Provincial Key Laboratory of Interdisciplinary Research and Application for Data Science, BNU-HKBU United International College. Zhuhai, 519087}
\and Shengxin ZHU\thanks{Research Center for Mathematics, Beijing Normal University, Zhuhai 519087, P.R. China; Guangdong Provincial Key Laboratory of Interdisciplinary Research and Application for Data Science, BNU-HKBU United International College, Zhuhai 519087, P.R. China (\email Shengxin.Zhu@bnu.edu.cn).}}

\begin{document}
\maketitle
 
\begin{abstract}
In this paper, we investigate diagonal estimation for large or implicit matrices, aiming to develop a novel and efficient stochastic algorithm that incorporates adaptive parameter selection. We explore the influence of different eigenvalue distributions on diagonal estimation and analyze the necessity of introducing the projection method and adaptive parameter optimization into the stochastic diagonal estimator. Based on this analysis, we derive a lower bound on the number of random query vectors needed to satisfy a given probabilistic error bound, which forms the foundation of our adaptive stochastic diagonal estimation algorithm. Finally, numerical experiments demonstrate the effectiveness of the proposed estimator for various matrix types, showcasing its efficiency and stability compared to other existing stochastic diagonal estimation methods.

\end{abstract}



\section{Introduction}\label{sec1}

For a matrix $\boldsymbol{A}\in \mathbb{R}^{n\times n}$, its trace is given by
\begin{equation} 
\mathrm{tr}(\boldsymbol{A}) = \sum_{i=1}^{n} [\boldsymbol{A}]_{i,i},
\end{equation}
where $[\boldsymbol{A}]_{i,i}$ denotes the $i$-th diagonal element of the matrix $\boldsymbol{A}$. This shows that by estimating the diagonal of $\boldsymbol{A}$, we can effectively approach the problem of matrix trace estimation, which is central to many applications, such as machine learning \cite{han2015large,2007Information}, scientific computing \cite{bai1996some,kalantzis2013accelerating}, statistics \cite{aune2014parameter}, computational biology \cite{estrada2000characterization}, computational physics \cite{andoni2015sketching}, etc.

Estimating the diagonal of a matrix not only facilitates the trace calculation but also has extensive applications across various scientific and engineering disciplines. For example, in statistics, leverage scores in subset selection can be derived from the diagonal of the projection matrix corresponding to the column space of the selected column vectors \cite{sobczyk2021estimating}. In risk management, the diagonal of the inverse covariance matrix is used to assess the confidence level in data quality \cite{bekas2009low}. In network science, subgraph centrality, which ranks the importance of network nodes, is measured by the diagonal of the matrix exponential of the adjacency matrix \cite{estrada2005subgraph}. In electronic structure calculations, the diagonal of the projection matrix corresponding to the minimum eigenvector of the Hamiltonian matrix is commonly computed \cite{bekas2007estimator}. Additionally, diagonal preconditioners are crucial in speeding up the convergence of iterative solvers for linear systems \cite{wathen2015preconditioning}. More recently, diagonal estimation has been applied to accelerate second-order optimization techniques in machine learning \cite{yao2021adahessian}. These applications highlight the critical role of diagonal estimation in solving numerous practical problems.

Standard methods for diagonal estimation include matrix decomposition techniques or evaluating quadratic forms, such as
\begin{equation} 
\label{eq:quadrature_estimator}
    [\boldsymbol{A}]_{i,i} = \boldsymbol{e}_{i}^{\mathrm{T}} \boldsymbol{A} \boldsymbol{e}_{i}, 
\end{equation}
where $\boldsymbol{e}_{i}\in \mathbb{R}^{n}$ is the $i$-th standard basis vector in $\mathbb{R}^{n}$. However, these approaches become computationally prohibitive or infeasible when the matrix $\boldsymbol{A}$ is large or implicit.

In recent years, stochastic methods have provide an efficient alternative, offering diagonal approximations through random sampling at significantly lower computational costs. Bekas et al. \cite{bekas2007estimator} introduced a Monte Carlo-based stochastic diagonal estimation method, defining the following stochastic diagonal estimator,
\begin{equation} \label{eq:diagonal_estimator}
    \mathrm{EST}_{\mathrm{diag}(\boldsymbol{A})}^{m} = \left[\sum_{i=1}^{m} \boldsymbol{\omega}^{(i)} \odot \boldsymbol{A} \boldsymbol{\omega}^{(i)}\right] \oslash \left[\sum_{i=1}^{m} \boldsymbol{\omega}^{(i)} \odot \boldsymbol{\omega}^{(i)}\right], \end{equation}
where $\boldsymbol{\omega}^{(i)}\in \mathbb{R}^{n}, i=1,2,\cdots,m$ are random query vectors\footnote{There are different types of query vectors, such as Rademacher random vectors, Gaussian random vectors, or sparse Rademacher random vectors.}, and $\mathrm{EST}_{\mathrm{diag}(\boldsymbol{A})}^{m}$ represents the estimate of $\mathrm{diag}(\boldsymbol{A})$ obtained using $m$ query vectors. Based on this method, several scalable approaches have been proposed \cite{laeuchli2016methods, kaperick2019diagonal, tang2012probing}, but more theoretical developments in this area are demanded.

Baston et al. \cite{baston2022stochastic} provided a theoretical analysis of the stochastic diagonal estimator defined in Eq. \eqref{eq:diagonal_estimator}, offering the first lower bound for the number of query vectors $m$ required to guarantee a desired relative error with a specified probability. Dharangutte et al. \cite{dharangutte2023tight} improved this result, presenting a more compact asymptotic bound for $m$. Recently, Hallman et al. \cite{hallman2023monte} examined stochastic diagonal estimators with query vectors from more general distributions, such as random vectors with bounded fourth moments or sparse Rademacher vectors with tunable sparsity, analyzing their error distributions using matrix concentration inequalities. These findings provide a robust theoretical framework for further advancements in matrix diagonal estimation.

\emph{Contributions.}  This paper demonstrates, through both numerical experiments and theoretical analysis, that applying the projection method to stochastic diagonal estimation can effectively improve computational efficiency. We also derive a more compact lower bound for the number of random query vectors required to achieve specified error and probability thresholds. Furthermore, we propose a new stochastic diagonal estimation algorithm that adaptively adjusts the size of the projection subspace and the number of query vectors, making it more suitable for estimating diagonals of different matrix types in practice.

\emph{Organization.} 
The paper is organized as follows: \Cref{sec:section2} provides theoretical and experimental insights into how projection-based methods enhance diagonal estimation. \Cref{sec:section3} focuses on the design of a stochastic diagonal estimation algorithm with adaptive parameter selection. In \Cref{sec:section4}, we evaluate the performance of the proposed algorithm. Finally, we summarize the key results in the conclusion.

\section{Diagonal estimation based on projection}
\label{sec:section2}
In this section, we elucidate the importance of projection methods in estimating matrix diagonal.

For any matrix $\boldsymbol{A}\in \mathbb{R}^{n\times n}$, Baston and Nakatsukasa \cite{baston2022stochastic} have proven that with the error tolerance $0<\varepsilon <1$ and probability $0<\delta <1$, when the number of query vectors $m$ in Eq. \eqref{eq:diagonal_estimator} satisfies 
\begin{equation}
\label{eq:m_bound_by_Yuji}
    m=O\left(\frac{\log(n/\delta)}{\varepsilon^{2}}\right),
\end{equation}
the following probability inequality holds
\begin{equation}
\label{eq:Yuji_result}
    \mathbb{P}\left\{\|\mathrm{EST}_{\mathrm{diag}(\boldsymbol{A})}^{m}-\mathrm{diag}(\boldsymbol{A})\|_{2}\leq \varepsilon\|\boldsymbol{A}_{\text{off}}\|_{F}\right\}\geq 1-\delta,
\end{equation}
where $\boldsymbol{A}_{\text{off}}$ represents the matrix obtained by setting all the diagonal of $\boldsymbol{A}$ to zero, i.e., $\|\boldsymbol{A}_{\text{off}}\|_{F}^{2}=\|\boldsymbol{A}\|_{F}^{2}-\|\mathrm{diag}(\boldsymbol{A})\|_{2}^{2}$.

\subsection{Preliminary Analysis}
\label{subsection:adaptive_diagonal_experiment}
Consider a symmetric positive semi-definite matrix $\boldsymbol{A}\in \mathbb{R}^{n\times n}$, whose eigenvalues form a vector $\boldsymbol{\lambda}$ in non-decreasing order, i.e., $\boldsymbol{\lambda}=\left[\lambda_{1},\lambda_{2},\cdots,\lambda_{n}\right]^{\mathrm{T}}$ and $\lambda_{1}\geq \lambda_{2}\geq \cdots \geq \lambda_{n}$. Thus, $\|\boldsymbol{A}\|_{F}^{2}=\sum_{i=1}^{n}\lambda_{i}^{2}=\|\boldsymbol{\lambda}\|_{2}^{2}$ and
\begin{align*}
	\left\|\mathrm{diag}(\boldsymbol{A})\right\|_{2}^{2} & =[\boldsymbol{A}]_{1,1}^{2}+[\boldsymbol{A}]_{2,2}^{2}+\ldots+[\boldsymbol{A}]_{n,n}^{2}  \\
	& \geq \frac{1}{n}\left(\sum_{i=1}^{n}[\boldsymbol{A}]_{i,i}\right)^{2} \\
	& = \frac{1}{n}\left(\mathrm{tr}(\boldsymbol{A})\right)^{2} = \frac{1}{n}\left\|\boldsymbol{\lambda}\right\|_{1}^{2}.
\end{align*}

Combining Eq. \eqref{eq:Yuji_result}, we can conclude that when the number of query vectors $m$ satisfies the condition in Eq. \eqref{eq:m_bound_by_Yuji}, the following holds
\begin{equation}
	\label{eq:diagonal_estimator_variant}
	\mathbb{P}\left\{\left\|\mathrm{EST}^{m}_{\mathrm{diag}(\boldsymbol{A})}-\mathrm{diag}(\boldsymbol{A})\right\|_{2}^{2}\leq \varepsilon^{2}\left(\left\|\boldsymbol{\lambda}\right\|_{2}^{2}-\frac{\left\|\boldsymbol{\lambda}\right\|_{1}^{2}}{n}\right) \right\}\geq 1-\delta.
\end{equation}
Without loss of generality, assuming the eigenvalues of matrix $\boldsymbol{A}$ are distributed within the interval $[a,b]$. For the same number of query vectors, one can roughly consider that the smaller the interval length $|b-a|$, the more concentrated the eigenvalue distribution is. In this case, the accuracy of the diagonal estimator in Eq. \eqref{eq:diagonal_estimator_variant} is also higher. 

To verify this, we consider the following examples. Let $n=1000$, and let $\boldsymbol{U}\in \mathbb{R}^{n\times n}$ be an arbitrary orthogonal matrix. We synthesize three matrices with different eigenvalue distribution intervals as follows \footnote{$\mathrm{randi}(\cdot)$ denotes the function for generating random integers in MATLAB.},
\begin{itemize}
	\item $\boldsymbol{A}_{1} = \boldsymbol{U}\boldsymbol{D}_{1}\boldsymbol{U}^{\mathrm{T}}$, where $\boldsymbol{D}_{1}=\mathrm{diag}(\mathrm{randi}([1800,2000],n,1)).$
	\item $\boldsymbol{A}_{2} = \boldsymbol{U}\boldsymbol{D}_{2}\boldsymbol{U}^{\mathrm{T}}$, where $\boldsymbol{D}_{2}=\mathrm{diag}(\mathrm{randi}([1000,2000],n,1)).$
	\item $\boldsymbol{A}_{3} = \boldsymbol{U}\boldsymbol{D}_{3}\boldsymbol{U}^{\mathrm{T}}$, where $\boldsymbol{D}_{3}=\mathrm{diag}(\mathrm{randi}([0,2000],n,1)).$
\end{itemize}

Gaussian random query vectors are generated and used in the stochastic diagonal estimator (Eq. \eqref{eq:diagonal_estimator}) to estimate the diagonal of the matrices $\boldsymbol{A}_{i},i=1,2,3$. We also measure the discrepancy between the estimated values and the exact values through the following formula,
\begin{equation*}
    \frac{\left\|\mathrm{EST}^{m}_{\mathrm{diag}(\boldsymbol{A}_{i})}-\mathrm{diag}(\boldsymbol{A}_{i})\right\|_{2}}{\left\|\mathrm{diag}(\boldsymbol{A}_{i})\right\|_{2}},\quad i \in \left\{ 1,2,3\right\}.
\end{equation*}

\begin{figure}[htp]
	\centering
	\includegraphics[width=.5\textwidth]{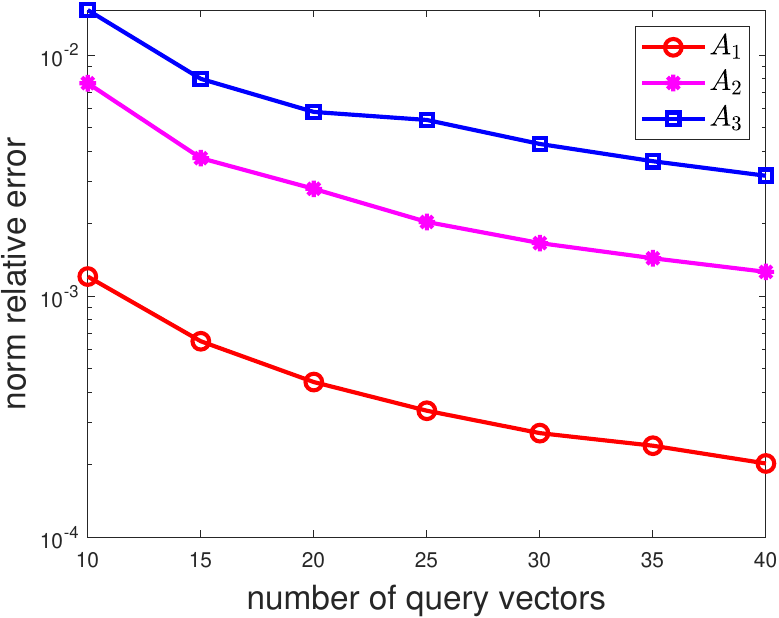}
	\caption{Impact of eigenvalue concentration on the accuracy of diagonal estimation based on Eq. \eqref{eq:diagonal_estimator}}
	\label{fig:spectral_concentrate_affect}
\end{figure}

The experimental results are presented in Figure~\ref{fig:spectral_concentrate_affect}. As the eigenvalues concentrate, the accuracy of the diagonal estimation improves accordingly. This observation leads to the natural approach of identifying an appropriate projection subspace, which projects out the portion of the matrix $\boldsymbol{A}$ associated with larger magnitudes of eigenvalues and then estimates the diagonal only for the remaining. That approach can be expressed as
\begin{equation}
	\label{eq:reducation_variance_diag}
	\mathrm{diag}(\boldsymbol{A}) = \mathrm{diag}(\boldsymbol{A}\boldsymbol{Q}\boldsymbol{Q}^{\mathrm{T}})+\mathrm{diag}(\boldsymbol{A}(\boldsymbol{I}-\boldsymbol{Q}\boldsymbol{Q}^{\mathrm{T}})),
\end{equation}
where $\boldsymbol{I}$ is the $n$-dimensional identity matrix, and $\boldsymbol{Q}\in \mathbb{R}^{n\times k}$ is column-orthogonal, meaning $\boldsymbol{Q}\boldsymbol{Q}^{\mathrm{T}}=\boldsymbol{I}$. We refer to the method, expressed in Eq. \eqref{eq:reducation_variance_diag}, as the projection-based diagonal estimation.

Since $\boldsymbol{Q}$ in Eq. \eqref{eq:reducation_variance_diag} is low-rank, the term $\mathrm{diag}(\boldsymbol{A}\boldsymbol{Q}\boldsymbol{Q}^{\mathrm{T}})$ can be computed efficiently with a small number of matrix-vector multiplications. This leaves only the estimation of the second term, $\mathrm{diag}(\boldsymbol{A}(\boldsymbol{I}-\boldsymbol{Q}\boldsymbol{Q}^{T}))$, which can be approximated by using the stochastic diagonal estimator given by Eq. \eqref{eq:diagonal_estimator}.

Next, we investigate how the decay rate of matrix eigenvalues affects the diagonal estimation using the projection method described in Eq. \eqref{eq:reducation_variance_diag}. Let the matrix $\boldsymbol{A}\in \mathbb{R}^{n\times n}$ be decomposed as $\boldsymbol{A}=\boldsymbol{U}\boldsymbol{D}\boldsymbol{U}^{\mathrm{T}}$, where $\boldsymbol{U}\in \mathbb{R}^{n\times n}$ is a randomly generated orthogonal matrix, and the diagonal matrix $\boldsymbol{D}\in \mathbb{R}^{n\times n}$ is constructed using four different decay profiles\footnote{These decay profiles are consistent with the matrix construction method in \cite{epperly2024xtrace}.}, with $n=1000$.
\begin{itemize}
    \item $flat$: $\boldsymbol{D}=\mathrm{diag}(3-2(i-1)/(n-1), \text{for}\,\,i= 1:n)$.
	\item $poly$: $\boldsymbol{D} = \mathrm{diag}(i^{-2}, \text{for}\,\,i=1:n)$.
	\item $exp$: $\boldsymbol{D}=\mathrm{diag}(0.7^{i}, \text{for}\,\,i=0:(n-1))$.
	\item $step$: $\boldsymbol{D}=\mathrm{diag}([\underbrace{1,\ldots,1}_{50},\underbrace{10^{-3},\ldots,10^{-3}}_{n-50}).$
\end{itemize}
Let $\boldsymbol{\Omega}_{1},\boldsymbol{\Omega}_{2},\boldsymbol{\Omega}_{3}\in \mathbb{R}^{n\times k}$ represent three Gaussian random matrices, with $k=20,40,60$ columns respectively. Define $ \boldsymbol{Q}_{i}=\mathrm{orth}(\boldsymbol{A}\boldsymbol{\Omega}_{i})$, so that $\boldsymbol{Q}_{i}$ forms a set of orthogonal bases for the column space of $\boldsymbol{A}\boldsymbol{\Omega}_{i}$.

For simplicity, we denote $\boldsymbol{A}(\boldsymbol{I}-\boldsymbol{Q}_{i}\boldsymbol{Q}_{i}^{\mathrm{T}})$ as $\boldsymbol{R}_{i}$ for $i=1,2,3$. We estimate $\boldsymbol{R}_{i}$ using the stochastic diagonal estimator defined in Eq. \eqref{eq:diagonal_estimator}, with the relative error measured by
\begin{equation*}
	\frac{\left\|\mathrm{EST}^{m}_{\mathrm{diag}(\boldsymbol{R}_{i})}-\mathrm{diag}(\boldsymbol{R}_{i})\right\|_{2}}{\left\|\mathrm{diag}(\boldsymbol{A})\right\|_{2}}.
\end{equation*}

\begin{figure}[ht]
    \subfigure[flat]{
    \begin{minipage}[t]{0.45\linewidth}
    \includegraphics[width=1\textwidth]{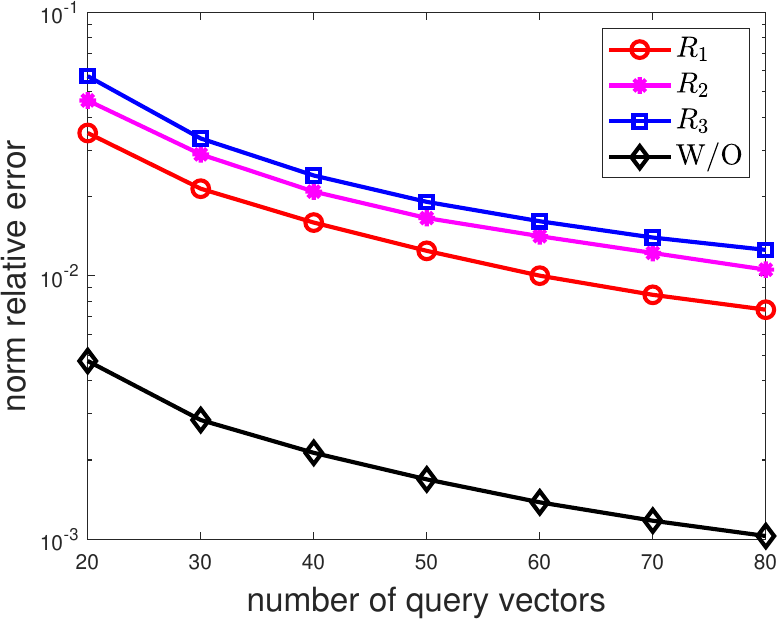}\label{fig:subfig-a}
    \end{minipage}%
    }%
    \hspace{1em}
    \subfigure[poly]{
    \begin{minipage}[t]{0.45\linewidth}
    \includegraphics[width=1\textwidth]{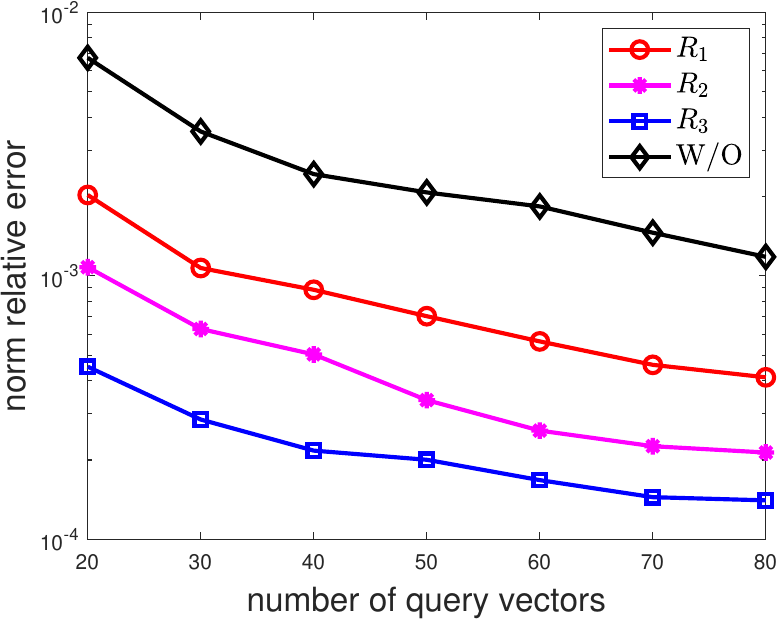}\label{fig:subfig-b}
    \end{minipage}%
    }%
    \vspace{1em}
    \subfigure[exp]{
    \begin{minipage}[t]{0.45\linewidth}
    \includegraphics[width=1\textwidth]{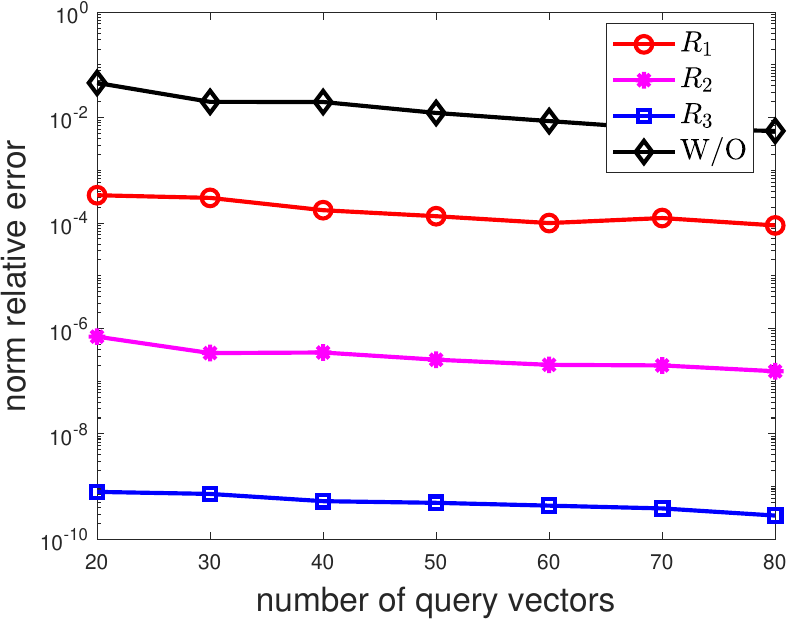}\label{fig:subfig-c}
    \end{minipage}
    }%
    \hspace{1em}
    \subfigure[step]{
    \begin{minipage}[t]{0.45\linewidth}
    \includegraphics[width=1\textwidth]{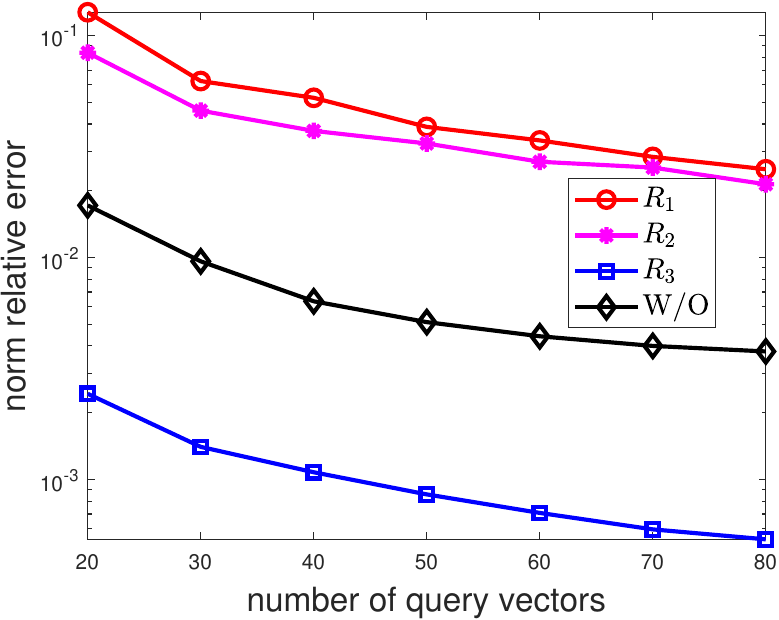}\label{fig:subfig-d}
    \end{minipage}
    }%
\centering
\caption{ Relative errors of diagonal estimation for matrices with  different eigenvalue distribution}
\label{fig:four_decay}
\end{figure}

Figure~\ref{fig:four_decay} presents the results of diagonal estimations for matrices with different eigenvalue decay rates. The label `W/O' represents the stochastic diagonal estimator from Eq. \eqref{eq:diagonal_estimator} used on matrix $\boldsymbol{A}$ directly. As seen in Figures~\ref{fig:subfig-b} and~\ref{fig:subfig-c}, when the matrix spectrum follows the $poly$ or $exp$ decay, the projection-based method significantly outperforms the stochastic diagonal estimator, with higher accuracy as the projection subspace dimension increases. 

Figure~\ref{fig:subfig-d} shows for matrices with a $step$ decay in their spectrum, the accuracy of the projection-based method depends on the projection subspace size and the number of query vectors. For example, matrices with a gap at the $50$-th eigenvalue exhibit improved accuracy only when the projection subspace dimension exceeds $50$ (as seen from $\boldsymbol{R}_{3}$). When the subspace dimension is smaller than $50$ (represented by $\boldsymbol{R}_{1}$ and $\boldsymbol{R}_{2}$), the projection-based method performs worse than the stochastic diagonal estimator. 

Conversely, as seen in Figure~\ref{fig:subfig-a}, when the matrix has a flat spectrum, the projection-based method's accuracy declines with an increasing subspace dimension, sometimes even performing worse than the stochastic diagonal estimator.

These findings suggest that the projection-based methods can enhance the accuracy of diagonal estimation in most cases. In certain scenarios, accuracy improves only when the projection subspace dimension exceeds a specific threshold. Thus, the size of projection subspace should be carefully tuned based on the matrix type to optimize diagonal estimation.

In the next section, we provide a theoretical analysis of how the projection-based methods may reduce the variance of stochastic diagonal estimators. Based on this, we will develop a stochastic diagonal estimator with adaptive parameter optimization.

\subsection{Theoretical Analysis}
Let us define $\boldsymbol{B}:=\boldsymbol{A}(\boldsymbol{I}-\boldsymbol{Q}\boldsymbol{Q}^{\mathrm{T}})$ in the projection-based matrix diagonal estimation defined in Eq. \eqref{eq:reducation_variance_diag}. By Eq. \eqref{eq:diagonal_estimator}, the diagonal of $\boldsymbol{B}$ can be estimated as,
\begin{equation}
	\label{eq:diagonal_B}
	\mathrm{EST}_{\mathrm{diag}(B)}^{m} = \left[\sum_{i=1}^{m}\boldsymbol{\omega}^{(i)}\odot \boldsymbol{B}\boldsymbol{\omega}^{(i)}\right]\oslash \left[\sum_{i=1}^{m}\boldsymbol{\omega}^{(i)}\odot \boldsymbol{\omega}^{(i)}\right].
\end{equation}

In the subsequent sections, we will theoretically analyze the projection-based stochastic diagonal estimators from two perspectives: one using the generalized stochastic diagonal estimator (defined below), and the other using the standard stochastic diagonal estimator (defined in Eq. \ref{eq:diagonal_B}).

\noindent{\textbf{Concentrated inequalities on generalized stochastic diagonal estimator}}

When the query vector $\boldsymbol{\omega}^{(i)}$ follows the Rademacher distribution, Eq. \eqref{eq:diagonal_B} is simplified to
\begin{equation}
\label{eq:general_diagonal_estimator}
\mathrm{EST}_{\mathrm{diag}(\boldsymbol{B})}^{m} = \frac{1}{m}\sum_{i=1}^{m}\boldsymbol{\omega}{(i)}\odot \boldsymbol{B}\boldsymbol{\omega}^{(i)}.
\end{equation}
Similarly, when the query vector $\boldsymbol{\omega}^{(i)}$ follows the standard Gaussian distribution, each term in $\sum_{i=1}^{m}\boldsymbol{\omega}^{(i)}\odot \boldsymbol{\omega}^{(i)}$ concentrates around $m$. Therefore, the diagonal estimator can be approximated by Eq. \eqref{eq:general_diagonal_estimator}. We refer to the estimator given in Eq. \eqref{eq:general_diagonal_estimator} as the 
\emph{generalized stochastic diagonal estimator}.

\begin{lemma}
\label{lemma:general_first_bound}
Let $\boldsymbol{B}\in \mathbb{R}^{n\times n}$ and let $\mathrm{EST}_{\mathrm{diag}(\boldsymbol{B})}^{m}$ be the estimation of $\mathrm{diag}(\boldsymbol{B})$ obtained by Eq. \eqref{eq:general_diagonal_estimator}, where the query vectors are Gaussian random vectors. For any $0<\varepsilon,\delta <1$, if the number of query vectors $m$ satisfies
\begin{equation}
	\label{eq:element_wise_m_bound}
	m\geq \frac{1}{\varepsilon^{2}\delta}\left(1+\frac{\left\|[\boldsymbol{B}]_{i,:}\right\|_{2}^{2}}{[\boldsymbol{B}]_{i,i}^{2}}\right),
\end{equation}
then 
\begin{equation}
	\mathbb{P}\left\{\left|\left(\mathrm{EST}_{\mathrm{diag}(\boldsymbol{B})}^{m}-\mathrm{diag}(\boldsymbol{B})\right)_{i}\right|\leq \varepsilon \left|[\boldsymbol{B}]_{i,i}\right|\right\}\geq 1-\delta.
\end{equation}
\end{lemma}
\begin{proof}
Consider the estimation error of the $i$-th diagonal element
	\begin{equation}
		\label{eq:single_element_error}
	\left(\mathrm{EST}_{\mathrm{diag}(\boldsymbol{B})}^{m}-\mathrm{diag}(\boldsymbol{B})\right)_{i}= \frac{1}{m}\sum_{k=1}^{m}\left(\boldsymbol{\omega}^{(k)}\odot \boldsymbol{B}\boldsymbol{\omega}^{(k)}-\mathrm{diag}(\boldsymbol{B})\right)_{i}.
	\end{equation}
Let $\boldsymbol{e}^{(k)}:= \boldsymbol{\omega}^{(k)}\odot \boldsymbol{B}\boldsymbol{\omega}^{(k)}-\mathrm{diag}(\boldsymbol{B})$, then
\begin{equation*}
	\left(\boldsymbol{e}^{(k)}\right)_{i} = [\boldsymbol{B}]_{i,i}\left(\boldsymbol{\omega}^{(k)}\right)_{i}^{2} -[\boldsymbol{B}]_{i,i} +\sum_{j\neq i}^{n}[\boldsymbol{B}]_{i,j}\left(\boldsymbol{\omega}^{(k)}\right)_{i}\left(\boldsymbol{\omega}^{(k)}\right)_{j}.
\end{equation*}
Furthermore, $\mathbb{E}[\left(\boldsymbol{e}^{(k)}\right)_{i}]=0$, $\mathbb{E}[\boldsymbol{e}^{(k)}]=\mathbf{0}$.

Let $\ell_{i} : = [\boldsymbol{B}]_{i,i}(\boldsymbol{\omega}^{(k)})_{i}^{2} -[\boldsymbol{B}]_{i,i}$. Since the fourth moment of the standard normal random variable is $3$, i.e., $\mathbb{E}[(\boldsymbol{\omega}^{(k)})_{i}^{4}]=3$, then $\mathbb{E}[\ell_{i}^{2}]=2[\boldsymbol{B}]_{i,i}^{2}$. The variance of $(\boldsymbol{e}^{(k)})_{i}$ is
\begin{align*}
	\mathbb{E}\left[\left(\boldsymbol{e}^{(k)}\right)_{i}^{2}\right]& = \mathbb{E}\left[\left(\ell_{i}+\sum_{j\neq i}^{n}[\boldsymbol{B}]_{i,j}\left(\boldsymbol{\omega}^{(k)}\right)_{i}\left(\boldsymbol{\omega}^{(k)}\right)_{j}\right)^{2}\right]\\
	& = \mathbb{E}[\ell_{i}^{2}] + \mathbb{E}\left[\sum_{j\neq i}^{n}\sum_{p\neq i}^{n}[\boldsymbol{B}]_{i,j}[\boldsymbol{B}]_{i,p}\left(\boldsymbol{\omega}^{(k)}\right)_{i}^{2}\left(\boldsymbol{\omega}^{(k)}\right)_{j}\left(\boldsymbol{\omega}^{(k)}\right)_{p}\right]\\
	& = 2[\boldsymbol{B}]_{i,i}^{2}+\sum_{j\neq i}^{n}[\boldsymbol{B}]_{i,j}^{2} = [\boldsymbol{B}]_{i,i}^{2}+\left\|[\boldsymbol{B}]_{i,:}\right\|_{2}^{2},
\end{align*}
where the second equation above holds because $$\mathbb{E}\left[2\ell_{i}\sum_{j\neq i}^{n}\left(\boldsymbol{\omega}^{(k)}\right)_{i}\left(\boldsymbol{\omega}^{(k)}\right)_{j}\right]=0,$$  and the third equation holds since the expected value is non-zero only when $p=j$.

Then, we have
\begin{align*}
	\mathbb{E}\left[\left(\mathrm{EST}^{m}_{\mathrm{diag}(\boldsymbol{B})}-\mathrm{diag}(\boldsymbol{B})\right)_{i}^{2}\right] & = \mathbb{E}\left[\frac{1}{m^{2}}\left(\sum_{k=1}^{m}\left(\boldsymbol{e}^{(k)}\right)_{i}\right)^{2}\right]\\
	&= \frac{1}{m^{2}}\mathbb{E}\left[\sum_{k=1}^{m}\left(\boldsymbol{e}^{(k)}\right)_{i}^{2}+\sum_{k=1}^{m}\sum_{q\neq k}^{m}\left(\boldsymbol{e}^{(k)}\right)_{i}\left(\boldsymbol{e}^{(q)}\right)_{i}\right]\\
	& =\frac{1}{m}\left([\boldsymbol{B}]_{i,i}^{2}+\left\|[\boldsymbol{B}]_{i,:}\right\|_{2}^{2}\right),
\end{align*}
where the third equation holds as $(\boldsymbol{e}^{(k)})_{i}$ and $(\boldsymbol{e}^{(q)})_{i}$ are independent when $q\neq k$.

Based on Markov's inequality,
\begin{equation}
	\label{eq:element_wise_error_bound}
	\mathbb{P}\left\{\left|\left(\mathrm{EST}_{\mathrm{diag}(\boldsymbol{B})}^{m}-\mathrm{diag}(\boldsymbol{B})\right)_{i}\right|\geq \varepsilon \left|[\boldsymbol{B}]_{i,i}\right|\right\}\leq \frac{1}{m\varepsilon^{2}}\left(1+\frac{\left\|[\boldsymbol{B}]_{i,:}\right\|_{2}^{2}}{[\boldsymbol{B}]_{i,i}^{2}}\right).
\end{equation}
Let 
$$\delta\geq \frac{1}{m\varepsilon^{2}}\left(1+\frac{\left\|[\boldsymbol{B}]_{i,:}\right\|_{2}^{2}}{[\boldsymbol{B}]_{i,i}^{2}}\right),$$ 
then the final conclusion holds.
\end{proof}

We further extend the probability bound of the estimation error for a single diagonal element, as established in Lemma~\ref{lemma:general_first_bound}, to the probability bound of the norm error for the entire diagonal vector estimation.

\begin{corollary}\label{corollary:general_first_bound}
	Let $\boldsymbol{B}\in \mathbb{R}^{n\times n}$ and let $\mathrm{EST}_{\mathrm{diag}(\boldsymbol{B})}^{m}$ be the estimation of $\mathrm{diag}(\boldsymbol{B})$ obtained by Eq. \eqref{eq:general_diagonal_estimator}, where the query vectors are Gaussian random vectors. For any $0<\varepsilon, \delta <1$, if the number of query vectors $m$ satisfies
	\begin{equation}
		\label{eq:norm_wise_m_bound}
		m\geq \frac{1}{\varepsilon^{2}\delta}\sum_{i=1}^{n}\left(1+\frac{\left\|[\boldsymbol{B}]_{i,:}\right\|_{2}^{2}}{[\boldsymbol{B}]_{i,i}^{2}}\right),
	\end{equation}
    then 
	\begin{equation}
		\mathbb{P}\left\{\left\|\mathrm{EST}_{\mathrm{diag}(\boldsymbol{B})}^{m}-\mathrm{diag}(\boldsymbol{B})\right\|_{2}\leq \varepsilon \left\|\mathrm{diag}(\boldsymbol{B})\right\|_{2}\right\}\geq 1-\delta.
	\end{equation}
\end{corollary}
\begin{proof}
	Based on union bound and Eq. \eqref{eq:element_wise_error_bound}, we have
	\begin{align*}
	&\mathbb{P}\left\{\left\|\mathrm{EST}^{m}_{\mathrm{diag}(\boldsymbol{B})}-\mathrm{diag}(\boldsymbol{B})\right\|_{2}\geq \varepsilon \left\|\mathrm{diag}(\boldsymbol{B})\right\|_{2}\right\} \\
	= &\mathbb{P}\left\{\left\|\mathrm{EST}^{m}_{\mathrm{diag}(\boldsymbol{B})}-\mathrm{diag}(\boldsymbol{B})\right\|_{2}^{2}\geq \varepsilon^{2}\left\|\mathrm{diag}(\boldsymbol{B})\right\|_{2}^{2}\right\}\\
		=& \mathbb{P}\left\{\sum_{i=1}^{n}\left(\mathrm{EST}^{m}_{\mathrm{diag}(\boldsymbol{B})}-\mathrm{diag}(\boldsymbol{B})\right)_{i}^{2}\geq \varepsilon^{2}\sum_{i=1}^{n}[\boldsymbol{B}]_{i,i}^{2}\right\}\\
		\leq & \sum_{i=1}^{n}\mathbb{P}\left\{\left(\mathrm{EST}^{m}_{\mathrm{diag}(\boldsymbol{B})}-\mathrm{diag}(\boldsymbol{B})\right)_{i}^{2}\geq \varepsilon^{2}[\boldsymbol{B}]_{i,i}^{2}\right\} \quad (\text{union bound})\\
		\leq & \frac{1}{m\varepsilon^{2}}\sum_{i=1}^{n}\left(1+\frac{\left\|[\boldsymbol{B}]_{i,:}\right\|_{2}^{2}}{[\boldsymbol{B}]_{i,i}^{2}}\right).
	\end{align*}
Let 
\begin{equation*}
	\delta \geq \frac{1}{m\varepsilon^{2}}\sum_{i=1}^{n}\left(1+\frac{\left\|[\boldsymbol{B}]_{i,:}\right\|_{2}^{2}}{[\boldsymbol{B}]_{i,i}^{2}}\right),
\end{equation*}
which leads to the final conclusion.
\end{proof}

By examining the lower bound on the number of query vectors in Eq. \eqref{eq:norm_wise_m_bound}, we observe that the greater the diagonal dominance of the matrix $\boldsymbol{B}$ is, the smaller the lower bound on $m$ is. Specifically, since $|[\boldsymbol{B}]_{i,i}|\leq \sigma_{\max}(\boldsymbol{B})$, the following inequality holds,
$$m\geq \frac{1}{\varepsilon^2\delta} \left(n+\frac{\left\|\boldsymbol{B}\right\|_{F}^{2}}{\sigma^{2}_{\max}(\boldsymbol{B})}\right),$$
implying that reducing the Frobenius norm of $\boldsymbol{B}$ can effectively decrease the lower bound on the number of query vectors. Additionally, since multiplying a matrix by an orthogonal projection matrix will not increase its Frobenius norm, we have $\|\boldsymbol{B}\|_{F}=\|\boldsymbol{A}(\boldsymbol{I}-\boldsymbol{Q}\boldsymbol{Q}^{\mathrm{T}})\|_{F}\leq \|\boldsymbol{A}\|_{F}$. From this perspective, the projection-based stochastic diagonal estimator can reduce the number of query vectors required to guarantee a given probabilistic error bound $(\varepsilon,\delta)$.

The results in Lemma~\ref{lemma:general_first_bound} and Corollary~\ref{corollary:general_first_bound} are derived based on the expected value of the quantity $(\mathrm{EST}_{\mathrm{diag}(\boldsymbol{B})}^{m}-\mathrm{diag}(\boldsymbol{B}))_{i}^{2}$ and the direct application of Markov's inequality. While these bounds are general, they tend to be somewhat loose. Inspired by the work of \cite{hallman2023monte}, which reformulates a vector problem concerning a matrix diagonal into a matrix problem involving a diagonal matrix, we extend the probabilistic error bound from the infinity norm of the matrix, $\|\mathrm{diag}(\mathrm{EST}_{\mathrm{diag}(\boldsymbol{B})}^{m})-\mathscr{D}(\boldsymbol{B})\|_{\infty}$ (as given in Corollary 4.2 of \cite{hallman2023monte}), to a more general bound in the vector $L_{2}$ norm, $\|\mathrm{EST}_{\mathrm{diag}(\boldsymbol{B})}^{m}-\mathrm{diag}(\boldsymbol{B})\|_{2}$. Here, $\mathscr{D}(\boldsymbol{B})$ denotes the diagonal matrix formed from the diagonal elements of $\boldsymbol{B}$.

\begin{lemma}\label{lemma:general_second_bound}
	Let $\boldsymbol{B}\in \mathbb{R}^{n\times n}$ and let $\mathrm{EST}_{\mathrm{diag}(\boldsymbol{B})}^{m}$ be the estimation of $\mathrm{diag}(\boldsymbol{B})$ obtained by Eq. \eqref{eq:general_diagonal_estimator}, where query vectors are Gaussian random vectors. For any $0<\varepsilon,\delta <1$, if the number of query vectors $m$ satisfies
	\begin{equation}
		\label{eq:norm_wise_m_bound_asymptotic}
		m = O\left(\frac{64(\mathrm{e}\log(n))^{3}}{\varepsilon^{2}\delta}\left\|\boldsymbol{B}\right\|_{F}^{2}\right),
	\end{equation}
    then
	\begin{equation}
		\mathbb{P}\left\{\left\|\mathrm{EST}_{\mathrm{diag}(\boldsymbol{B})}^{m}-\mathrm{diag}(\boldsymbol{B})\right\|_{2}\leq \varepsilon \right\}\geq 1-\delta.
	\end{equation}
\end{lemma}
\begin{proof}
	Based on Eq. \eqref{eq:single_element_error}, we have
	\begin{equation}
		\left\|\mathrm{EST}_{\mathrm{diag}(\boldsymbol{B})}^{m}-\mathrm{diag}(\boldsymbol{B})\right\|_{2}^{2} = \frac{1}{m^{2}}\left\|\sum_{k=1}^{m}\boldsymbol{e}^{(k)}\right\|_{2}^{2}.
	\end{equation}
Let $r_{1},r_{2},\ldots,r_{m}$ denote independent Rademacher random vectors, as $f(\boldsymbol{x})=\|\boldsymbol{x}\|_{2}^{2}$ is convex and $\mathbb{E}\left[\sum_{k=1}^{m}r_{k}\mathrm{diag}(\boldsymbol{B})\right]=\mathbf{0}$, combining Lem. 6.1.2 and Lem. 6.4.2 in \cite{vershynin2018high}, we have
\begin{align*}
	\mathbb{E}\left[\left\|\sum_{k=1}^{m}\boldsymbol{e}^{(k)}\right\|_{2}^{2}\right] & \leq 2\mathbb{E}\left[\left\|\sum_{k=1}^{m}r_{k}\boldsymbol{e}^{(k)}\right\|_{2}^{2}\right]\\
	& \leq 2\mathbb{E}\left[\left\|\sum_{k=1}^{m}r_{k}\boldsymbol{e}^{(k)}+\sum_{k=1}^{m}r_{k}\mathrm{diag}(\boldsymbol{B})\right\|_{2}^{2}\right] \\
	& = 2\mathbb{E}\left[\left\|\sum_{k=1}^{m}\boldsymbol{\omega}^{(k)}\odot r_{k}\boldsymbol{B}\boldsymbol{\omega}^{(k)}\right\|_{2}^{2}\right].
\end{align*}
Let $\boldsymbol{W}_{k}:=\mathrm{diag}(\boldsymbol{\omega}^{(k)})$, $\boldsymbol{D}_{k} := \mathrm{diag}(r_{k}\boldsymbol{B}\boldsymbol{\omega}^{(k)})$, then
\begin{align*}
	\left\|\sum_{k=1}^{m}\boldsymbol{\omega}^{(k)}\odot r_{k}\boldsymbol{B}\boldsymbol{\omega}^{(k)}\right\|_{2}^{2} & = \sum_{i=1}^{n}\left(\sum_{k=1}^{m}\left(\boldsymbol{\omega}^{(k)}\right)_{i}\left(r_{k}\boldsymbol{B}\boldsymbol{\omega}^{(k)}\right)_{i}\right)^{2} \\
	& \leq n\max_{1\leq i\leq n}\left(\sum_{k=1}^{m}\left(\boldsymbol{\omega}^{(k)}\right)_{i}\left(r_{k}\boldsymbol{B}\boldsymbol{\omega}^{(k)}\right)_{i}\right)^{2}\\
	& = n \left\|\sum_{k=1}^{m}\boldsymbol{W}_{k}\boldsymbol{D}_{k}\right\|_{2}^{2}.
\end{align*}
Thus, $\mathbb{E}[\|\sum_{k=1}^{m}\boldsymbol{e}^{(k)}\|_{2}^{2}]\leq 2 n \mathbb{E}[\|\sum_{k=1}^{m}\boldsymbol{W}_{k}\boldsymbol{D}_{k}\|_{2}^{2}]$. Since $r_{k}$ and $\boldsymbol{\omega}^{(k)}$ are independent, $\mathbb{E}\left[\boldsymbol{W}_{k}\boldsymbol{D}_{k}\right]=\mathbf{O}$. Based on Thm. 3.2 in \cite{chen2012masked}, we have
\begin{equation}
	\label{eq:first_second_term}
	\begin{aligned}
		\left(\mathbb{E}\left[\left\|\sum_{k=1}^{m}\boldsymbol{W}_{k}\boldsymbol{D}_{k}\right\|_{2}^{2}\right]\right)^{1/2}
		\leq\  & \sqrt{2\mathrm{e}\log(n)}\underbrace{\left\|\left(\sum_{k=1}^{m}\mathbb{E}\left[(\boldsymbol{W}_{k}\boldsymbol{D}_{k})^{2}\right]\right)^{1/2}\right\|_{2}}_{\text{first term}}+\\
  &\quad\ 4\mathrm{e}\log(n)\underbrace{\left(\mathbb{E}\left[\max_{1\leq k\leq m}\left\|\boldsymbol{W}_{k}\boldsymbol{D}_{k}\right\|_{2}^{2}\right]\right)^{1/2}.}_{\text{second term}}
	\end{aligned}
\end{equation}
For Gaussian random vector $\boldsymbol{\omega}^{(k)}\sim \mathcal{N}(\mathbf{0},\boldsymbol{I})$, the following relationship holds \cite{chen2012masked},
\begin{equation}
	\label{eq:the_fourthmoment_of_gaussian_vector}
	\left(\mathbb{E}\left[\max_{1\leq k\leq m}\left\|\boldsymbol{\omega}^{(k)}\right\|_{\infty}^{4}\right]\right)^{1/2}\leq \mathrm{e}\log(nm).
\end{equation}
\begin{itemize}
	\item The first term: The $i$-th diagonal element of matrix $\left(\boldsymbol{W}_{k}\boldsymbol{D}_{k}\right)^{2}$ is 
		\begin{align*}
			\left[\boldsymbol{W}_{k}^{2}\boldsymbol{D}_{k}^{2}\right]_{i,i} &= \left(\boldsymbol{\omega}^{(k)}\right)_{i}^{2}\left(\sum_{j=1}^{n}[\boldsymbol{B}]_{i,j}(\boldsymbol{\omega}^{(k)})_{j}\right)^{2}\\
			&\leq \left\|\boldsymbol{\omega}^{(k)}\right\|_{\infty}^{4}\left\|[\boldsymbol{B}]_{i,:}\right\|_{1}^{2}.
		\end{align*}
	As $\boldsymbol{W}_{k}\boldsymbol{D}_{k}$ is diagonal matrix,
	\begin{equation*}
		\boldsymbol{W}_{k}^{2}\boldsymbol{D}_{k}^{2}\preceq \left\|\boldsymbol{\omega}^{(k)}\right\|_{\infty}^{4}\left[\begin{array}{cccc}
			\left\|[\boldsymbol{B}]_{1,:}\right\|_{1}^{2} & & & \\
			& \left\|[\boldsymbol{B}]_{2,:}\right\|_{1}^{2} & & \\
			& & \ddots & \\
			& &  & \left\|[\boldsymbol{B}]_{n,:}\right\|_{1}^{2}
		\end{array}\right].
	\end{equation*}
Combining this equality with Eq. \eqref{eq:the_fourthmoment_of_gaussian_vector}, we have
\begin{equation}
	\begin{aligned}
		\label{eq:the_first_term}
  \left(\sum_{k=1}^{m}\mathbb{E}\left[\boldsymbol{W}_{k}^{2}\boldsymbol{D}_{k}^{2}\right]\right)^{1/2} &\preceq \left(m\max_{1\leq i\leq n}\left\|[\boldsymbol{B}]_{i,:}\right\|_{1}^{2}\boldsymbol{I}\mathbb{E}\left[\max_{1\leq k\leq m}\left\|\boldsymbol{\omega}^{(k)}\right\|_{\infty}^{4}\right] \right)^{1/2}\\
		& = \sqrt{m}\mathrm{e}\log(nm)\left\|\boldsymbol{B}\right\|_{\infty}\boldsymbol{I}.
	\end{aligned}
\end{equation} 
Thus, the first term in Eq. \eqref{eq:first_second_term} is not larger than $\sqrt{m}\mathrm{e}\log(nm)\|\boldsymbol{B}\|_{\infty}$.
	\item The second term: As $[\boldsymbol{W}_{k}\boldsymbol{D}_{k}]_{i,i}\leq \|[\boldsymbol{B}]_{i,:}\|_{1}\|\boldsymbol{\omega}^{(k)}\|_{\infty}^{2}$, we have $\|\boldsymbol{W}_{k}\boldsymbol{D}_{k}\|_{2}^{2}\leq \|\boldsymbol{B}\|_{\infty}^{2}\|\boldsymbol{\omega}^{(k)}\|_{\infty}^{4}$, then
	\begin{equation}
        \label{eq:the_second_term}
		\left(\mathbb{E}\left[\max_{1\leq k\leq m}\left\|\boldsymbol{W}_{k}\boldsymbol{D}_{k}\right\|_{2}^{2}\right]\right)^{1/2} \leq \mathrm{e}\log(nm)\|\boldsymbol{B}\|_{\infty}.
	\end{equation}
\end{itemize}
Based on Eq. \eqref{eq:the_first_term} and Eq. \eqref{eq:the_second_term},
\begin{equation}
	\left(\mathbb{E}\left[\left\|\sum_{k=1}^{m}\boldsymbol{W}_{k}\boldsymbol{D}_{k}\right\|_{2}^{2}\right]\right)^{1/2} \leq (\sqrt{2\mathrm{e}m\log(n)}+4\mathrm{e}\log(n))\mathrm{e}\log(nm)\|\boldsymbol{B}\|_{\infty}.
\end{equation}
Furthermore,
\begin{equation}
	\begin{aligned}
		\mathbb{E}\left[\left\|\mathrm{EST}^{m}_{\mathrm{diag}(\boldsymbol{B})}-\mathrm{diag}(\boldsymbol{B})\right\|_{2}^{2}\right] & = \frac{1}{m^{2}}\mathbb{E}\left[\left\|\sum_{k=1}^{m}\boldsymbol{e}^{(k)}\right\|_{2}^{2}\right]\\
		&\leq \frac{2n}{m^{2}}\mathbb{E}\left[\left\|\sum_{k=1}^{m}\boldsymbol{W}_{k}\boldsymbol{D}_{k}\right\|_{2}^{2}\right]\\
		& \leq \frac{n}{2}\left(\sqrt{\frac{8\mathrm{e}\log(n)}{m}}+\frac{8\mathrm{e}\log(n)}{m}\right)^{2}\left(\mathrm{e}\log(nm)\right)^{2}\left\|\boldsymbol{B}\right\|_{\infty}^{2}.
	\end{aligned}
\end{equation}
When $8\mathrm{e}\log(n)\leq m \leq n$, the above result can be transformed into
\begin{equation}
	\mathbb{E}\left[\left\|\mathrm{EST}^{m}_{\mathrm{diag}(\boldsymbol{B})}-\mathrm{diag}(\boldsymbol{B})\right\|_{2}^{2}\right]\leq \frac{64n(\mathrm{e}\log(n))^{3}}{m}\left\|\boldsymbol{B}\right\|_{\infty}^{2}.
\end{equation} 
As $n\|\boldsymbol{B}\|_{\infty}^{2}\geq \|\boldsymbol{B}\|_{F}^{2}$, combining Markov inequality, we have that if the number of random query vectors $m$ satisfies
\begin{equation*}
	m = O\left(\frac{64(\mathrm{e}\log(n))^{3}}{\varepsilon^{2}\delta}\left\|\boldsymbol{B}\right\|_{F}^{2}\right),
\end{equation*}
then 
\begin{equation*}
	\mathbb{P}\left\{\left\|\mathrm{EST}_{\mathrm{diag}(\boldsymbol{B})}^{m}-\mathrm{diag}(\boldsymbol{B})\right\|_{2}\leq \varepsilon\right\}\geq 1-\delta.
\end{equation*}
\end{proof}

The lower bound on the number of query vectors $m$ in Eq. \eqref{eq:norm_wise_m_bound_asymptotic} further underscores the advantage of using a projection-based stochastic estimator in reducing the number of query vectors needed to achieve a specified accuracy. Specifically, performing stochastic diagonal estimation on the projected matrix $\boldsymbol{A}(\boldsymbol{I}-\boldsymbol{Q}\boldsymbol{Q}^{\mathrm{T}})$ requires fewer query vectors to reach the same error tolerance than applying stochastic diagonal estimation directly to $\boldsymbol{A}$.

\noindent{\textbf{Concentrated inequalities on stochastic diagonal estimator}}

The generalized stochastic diagonal estimator in Eq. \eqref{eq:general_diagonal_estimator}, which represents a form of the stochastic diagonal estimator (Eq. \eqref{eq:diagonal_B}) in expectation, provides a straightforward structure that facilitates theoretical analysis. As discussed earlier, the projection-based diagonal estimator reduces the number of required query vectors. However, as emphasized in \Cref{subsection:adaptive_diagonal_experiment}, an estimator that can adaptively adjust the projection subspace size is essential for optimizing efficiency. Therefore, it becomes crucial to analyze the standard form of the stochastic diagonal estimator and establish a precise, non-asymptotic lower bound for the number of query vectors.

\begin{theorem}
	\label{theorem:our_diagonal_bound}
	Let $\boldsymbol{B}\in \mathbb{R}^{n\times n}$ and let $\mathrm{EST}_{\mathrm{diag}(\boldsymbol{B})}^{m}$ be the estimation of $\mathrm{diag}(\boldsymbol{B})$ obtained by \eqref{eq:diagonal_B}, where query vectors are Gaussian random vectors. For any $0<\varepsilon,\delta<1$, if the number of query vectors $m$ satisfies
	\begin{equation}
		\label{eq:m_bound_our}
		m \geq 1+2\dfrac{\log\left(\sqrt{\dfrac{2}{\uppi}}\dfrac{n\left\|\boldsymbol{B}_{\text{off}}\right\|_{F}}{\varepsilon\delta}\right)}{\log\left(1+\dfrac{\varepsilon^{2}}{\left\|\boldsymbol{B}_{\text{off}}\right\|_{F}^{2}}\right)},
	\end{equation}
where $\boldsymbol{B}_{\text{off}}$ is the matrix obtained by setting all the diagonal of $\boldsymbol{B}$ to zero, then
\begin{equation}
	\mathbb{P}\left\{\left\|\mathrm{EST}_{\mathrm{diag}(\boldsymbol{B})}^{m}-\mathrm{diag}(\boldsymbol{B})\right\|_{2}\leq \varepsilon\right\}\geq 1-\delta.
\end{equation}
\end{theorem}
\begin{proof}
	Based on Thm. 5.8 in \cite{hallman2023monte}, we have that for any $1\leq i \leq n$,
	\begin{equation}
		\begin{aligned}
		&\mathbb{P}\left\{\left|\left(\mathrm{EST}_{\mathrm{diag}(\boldsymbol{B})}^{m}\right)_{i}-[\boldsymbol{B}]_{i,i}\right|\geq t\right\}\\
		\leq & \frac{1}{t}\sqrt{\frac{2\left(\left\|[\boldsymbol{B}]_{i,:}\right\|_{2}^{2}-[\boldsymbol{B}]_{i,i}^{2}\right)}{\uppi m}}\left(1+\frac{t^{2}}{\left\|[\boldsymbol{B}]_{i,:}\right\|_{2}^{2}-[\boldsymbol{B}]_{i,i}^{2}}\right)^{-\frac{m-1}{2}}.
		\end{aligned}
	\end{equation}
Let $t^{2} = \varepsilon^{2}(\|[\boldsymbol{B}]_{i,:}\|_{2}^{2}-[\boldsymbol{B}]_{i,i}^{2})$, and substitute it into the above inequality,
\begin{align*}
	& \mathbb{P}\left\{\left|\left(\mathrm{EST}_{\mathrm{diag}(\boldsymbol{B})}^{m}\right)_{i}-[\boldsymbol{B}]_{i,i}\right|^{2}\geq\varepsilon^{2}\left(\left\|[\boldsymbol{B}]_{i,:}\right\|_{2}^{2}-[\boldsymbol{B}]_{i,i}^{2}\right)\right\} \\
	\leq & \frac{\sqrt{2}}{\varepsilon\sqrt{\uppi m}}\left(1+\varepsilon^{2}\right)^{-\frac{m-1}{2}} 
	\leq  \sqrt{\frac{2}{\uppi}}\frac{\left(1+\varepsilon^{2}\right)^{-\frac{m-1}{2}}}{\varepsilon}. \quad (\text{as}\,\, m\geq 1)
\end{align*}
Applying the union bound, we have
\begin{align*}
	& \mathbb{P}\left\{\left\|\mathrm{EST}_{\mathrm{diag}(\boldsymbol{B})}^{m}-\mathrm{diag}(\boldsymbol{B})\right\|_{2}^{2}\geq \varepsilon^{2}\left(\left\|\boldsymbol{B}\right\|_{F}^{2}-\left\|\mathrm{diag}(\boldsymbol{B})\right\|_{2}^{2}\right)\right\} \\
	\leq & \sum_{i=1}^{n}\mathbb{P}\left\{\left|\left(\mathrm{EST}_{\mathrm{diag}(\boldsymbol{B})}^{m}\right)_{i}-[\boldsymbol{B}]_{i,i}\right|\geq\varepsilon^{2}\left(\left\|[\boldsymbol{B}]_{i,:}\right\|_{2}^{2}-[\boldsymbol{B}]_{i,i}^{2}\right)\right\} \\
	\leq & \sqrt{\frac{2}{\uppi}}\frac{n\left(1+\varepsilon^{2}\right)^{-\frac{m-1}{2}}}{\varepsilon}.
\end{align*}
Let $\delta\geq \sqrt{\frac{2}{\uppi}}\frac{n}{\varepsilon}\left(1+\varepsilon^{2}\right)^{-\frac{m-1}{2}}$, then
\begin{equation*}
	m\geq 1+2\dfrac{\log\left(\sqrt{\dfrac{2}{\uppi}}\dfrac{n}{\varepsilon\delta}\right)}{\log(1+\varepsilon^{2})},
\end{equation*}
which guarantees the following inequality,
\begin{equation}
	\mathbb{P}\left\{\left\|\mathrm{EST}_{\mathrm{diag}(\boldsymbol{B})}^{m}-\mathrm{diag}(\boldsymbol{B})\right\|_{2}\geq \varepsilon\left\|\boldsymbol{B}_{\text{off}}\right\|_{F}\right\}.
\end{equation}
Finally, one may let $\varepsilon = \varepsilon\|\boldsymbol{B}_{\text{off}}\|_{F}$, and substitute it into the above probability inequality to obtain the conclusion.
\end{proof}

Based on this theoretical result, we will next design a stochastic diagonal estimator with adaptive parameter optimization.

\section{Adaptive diagonal estimator}
\label{sec:section3}
Let $\boldsymbol{A}\in \mathbb{R}^{n\times n}$ be a general matrix. Recall the estimator for the diagonal of the matrix based on the projection method described in Eq. \eqref{eq:reducation_variance_diag},
\begin{equation*}
	\mathrm{diag}(\boldsymbol{A}) = \underbrace{\mathrm{diag}(\boldsymbol{A}\boldsymbol{Q}\boldsymbol{Q}^{\mathrm{T}})}_{\boldsymbol{d}_{\mathrm{defl}}}+\underbrace{\mathrm{diag}(\boldsymbol{A}(\boldsymbol{I}-\boldsymbol{Q}\boldsymbol{Q}^{\mathrm{T}}))}_{\boldsymbol{d}_{\mathrm{rem}}},
\end{equation*}
where $\boldsymbol{Q}\in \mathbb{R}^{n\times k}$ and $\boldsymbol{Q}^{\mathrm{T}}\boldsymbol{Q}=\boldsymbol{I}$. We denote the value or estimation of $\mathrm{diag}(\boldsymbol{A}\boldsymbol{Q}\boldsymbol{Q}^{\mathrm{T}})$ and $\mathrm{diag}(\boldsymbol{A}(\boldsymbol{I}-\boldsymbol{Q}\boldsymbol{Q}^{\mathrm{T}}))$ by $\boldsymbol{d}_{\mathrm{defl}}$ and $\boldsymbol{d}_{\mathrm{rem}}$, respectively.

\begin{algorithm}
	\caption{Prototype algorithm for matrix $\boldsymbol{A}$ (based on RSVD)}
	\label{algorithm:Prototype_algorithm_based_rsvd}
	\begin{algorithmic}[1]
		\STATE{\textbf{Input} Matrix $\boldsymbol{A}\in \mathbb{R}^{n\times n}$, error tolerance $\varepsilon>0$, failure probability $\delta\in (0,1)$}
		\STATE{\textbf{Initialize} Gaussian random matrix $\boldsymbol{\Omega}\in \mathbb{R}^{n\times k}$, $\boldsymbol{mole}=\mathrm{zeros}(n,1)$, $\boldsymbol{deno}=\mathrm{zeros}(n,1)$}
		\STATE{$\boldsymbol{Y}=\boldsymbol{A}\boldsymbol{\Omega}$}
		\STATE{Obtain orthonormal basis $\boldsymbol{Q}$ for $\mathrm{range}(\boldsymbol{Y})$}
		\STATE{$\boldsymbol{d}_{\mathrm{defl}}=\mathrm{diag}((\boldsymbol{A}\boldsymbol{Q})\boldsymbol{Q}^{\mathrm{T}})$}
		\STATE{Compute query vectors number $m$ based on Eq. \eqref{eq:m_bound_our}}
		\FOR{$i=1$ to $m$}
		\STATE{Sample Gaussian random vector $\boldsymbol{\omega}_{i}\in \mathbb{R}^{n}$}
		\STATE{$\boldsymbol{y}_{i}=(\boldsymbol{I}-\boldsymbol{Q}\boldsymbol{Q}^{\mathrm{T}})\boldsymbol{\omega}_{i}$}
		\STATE{Update $\boldsymbol{mole}=\boldsymbol{mole}+\boldsymbol{\omega}_{i}\odot(\boldsymbol{A}\boldsymbol{y}_{i})$, and update $\boldsymbol{deno}=\boldsymbol{deno}+\boldsymbol{\omega}_{i}
			\odot\boldsymbol{\omega}_{i}$}
		\ENDFOR
		\STATE{$\boldsymbol{d}_{\mathrm{rem}}=\boldsymbol{mole}\oslash\boldsymbol{deno}$}
		\RETURN{$\boldsymbol{d}_{\mathrm{defl}}+\boldsymbol{d}_{\mathrm{rem}}$}
	\end{algorithmic}
\end{algorithm}

Algorithm \ref{algorithm:Prototype_algorithm_based_rsvd} presents a prototype for a stochastic diagonal estimation algorithm utilizing the projection-based approach. The algorithm begins by multiplying a Gaussian random matrix $\boldsymbol{\Omega}\in \mathbb{R}^{n\times k}(k<n)$ to the right of matrix $\boldsymbol{A}\in \mathbb{R}^{n\times n}$, and subsequently computes the orthogonal basis of $\mathrm{range}(\boldsymbol{A}\boldsymbol{\Omega})$, resulting in the column orthogonal matrix $\boldsymbol{Q}\in \mathbb{R}^{n\times k}$. This is a common approach in randomized SVD (RSVD). The $\mathrm{diag}((\boldsymbol{A}\boldsymbol{Q})\boldsymbol{Q}^{\mathrm{T}})$ is then computed directly through matrix-vector multiplication, as outlined in lines $3$ to $5$ of Algorithm \ref{algorithm:Prototype_algorithm_based_rsvd}. However, since $\boldsymbol{A}(\boldsymbol{I}-\boldsymbol{Q}\boldsymbol{Q}^{\mathrm{T}})$ is still high dimension, it cannot be computed directly. Instead, it is estimated by the stochastic diagonal estimator defined in Eq. \eqref{eq:diagonal_B}, as detailed in lines $7$ to $12$ of Algorithm \ref{algorithm:Prototype_algorithm_based_rsvd}.

To improve the efficiency of estimating $\mathrm{diag}(\boldsymbol{A})$, we focus on minimizing the number of matrix-vector multiplications with respect to matrix $\boldsymbol{A}$ in Algorithm~\ref{algorithm:Prototype_algorithm_based_rsvd}. The total number of matrix-vector multiplications needed for Algorithm~\ref{algorithm:Prototype_algorithm_based_rsvd}~ can be derived by referring to Theorem~\ref{theorem:our_diagonal_bound}. Given the probabilistic error bound parameters $(\varepsilon,\delta)$, let $\mathrm{EST}_{\mathrm{diag}(\boldsymbol{A})}$ denote the final diagonal estimation from the projection-based method. To satisfy the following error bound
\begin{equation*}
\mathbb{P}\left\{\left\|\mathrm{EST}_{\mathrm{diag}(\boldsymbol{A})}-\mathrm{diag}(\boldsymbol{A})\right\|_{2}\leq \varepsilon \right\}\geq 1-\delta,
\end{equation*}
the total number of matrix-vector multiplications with $\boldsymbol{A}$ (denoted by $\# \mathrm{matvecs}$) is given by
\begin{equation}
\label{eq:matvecs_function}
\#\mathrm{matvecs}(k) = 2k + m,
\end{equation}
where $m$ is defined in Eq. \eqref{eq:m_bound_our}. Note that since the error mainly comes from estimating $\mathrm{diag}(\boldsymbol{A}(\boldsymbol{I}-\boldsymbol{Q}\boldsymbol{Q}^{\mathrm{T}}))$ and $\mathrm{diag}(\boldsymbol{A}\boldsymbol{Q}\boldsymbol{Q}^{\mathrm{T}})$ can be exactly computed, the matrix $\boldsymbol{B}$ in Eq. \eqref{eq:m_bound_our} refers to $\boldsymbol{A}(\boldsymbol{I}-\boldsymbol{Q}\boldsymbol{Q}^{\mathrm{T}})$. 

Since $m$ is an increasing function of $\|\boldsymbol{B}_{\text{off}}\|_{F}$ and is affected by the choice of $k$, we simplify $m$ as a function of $\|\boldsymbol{B}_{\text{off}}\|_{F}$ (i.e., $g(\|\boldsymbol{B}_{\text{off}}\|_{F})$), then
\begin{equation}
	\label{eq:matvecs_function_2}
	\#\mathrm{matvecs}(k) = 2k + g\left(\left\|\boldsymbol{B}_{\text{off}}\right\|_{F}\right),
\end{equation} 
where
\begin{equation*}
    g\left(\left\|\boldsymbol{B}_{\text{off}}\right\|_{F}\right) = 1+2\dfrac{\log\left(\sqrt{\dfrac{2}{\uppi}}\dfrac{n\left\|\boldsymbol{B}_{\text{off}}\right\|_{F}}{\varepsilon\delta}\right)}{\log\left(1+\dfrac{\varepsilon^{2}}{\left\|\boldsymbol{B}_{\text{off}}\right\|_{F}^{2}}\right)}.
\end{equation*}

Ideally, the column orthogonal matrix $\boldsymbol{Q}$ consists of the first $k$ principal eigenvectors of $\boldsymbol{A}$. Figure~\ref{fig:min_matvecs} illustrates the variation of $\#\mathrm{matvecs}$ with respect to the size of the projection subspace $k$ under different eigenvalue decay patterns. It is shown that the minimum value of $\#\mathrm{matvecs}$ exists across all eigenvalue decay patterns, though the optimal value of $k$ varies depending on the matrix type. This leads to the first challenge in adaptive stochastic diagonal estimation: determining the value of $k$ that minimizes $\#\mathrm{matvecs}$, i.e., identifying the optimal number of columns of $\boldsymbol{Q}$ to minimize the total number of matrix-vector multiplications in the algorithm.

\begin{figure}[htbp]
    \subfigure[flat]{
    \begin{minipage}[t]{0.45\linewidth}
    \includegraphics[width=1\textwidth]{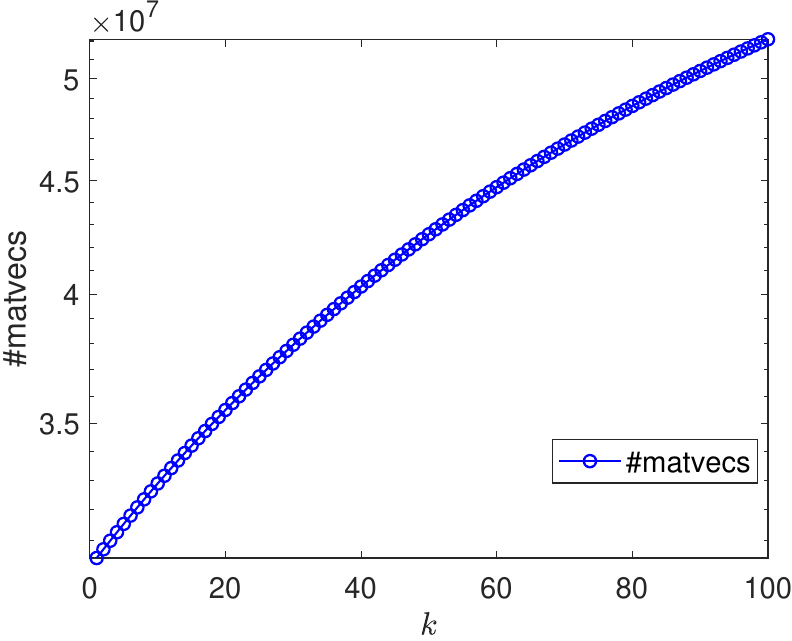}\label{fig:subfig-flag-min}
    \end{minipage}%
    }%
    \hspace{1em}
    \subfigure[poly]{
    \begin{minipage}[t]{0.45\linewidth}
    \includegraphics[width=1\textwidth]{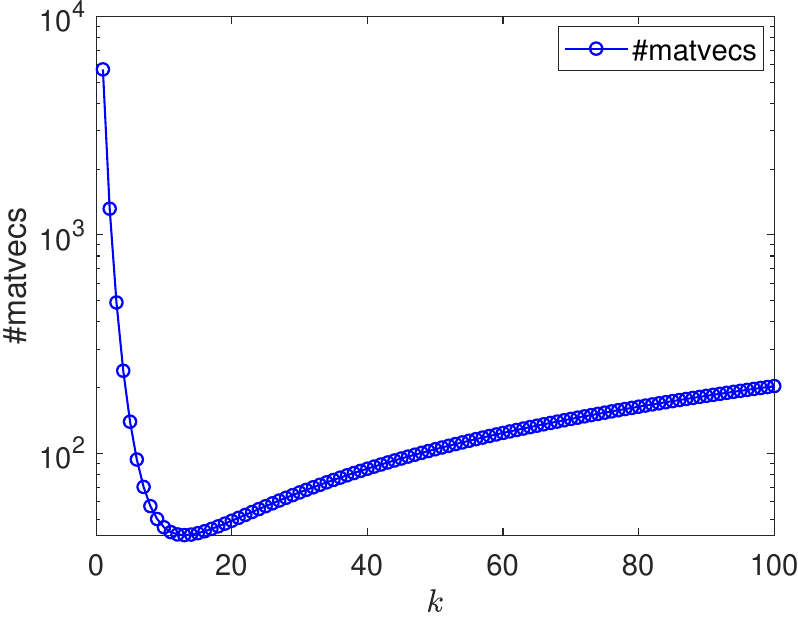}\label{fig:subfig-ploy-min}
    \end{minipage}%
    }%
    \vspace{1em}
    \subfigure[exp]{
    \begin{minipage}[t]{0.45\linewidth}
    \includegraphics[width=1\textwidth]{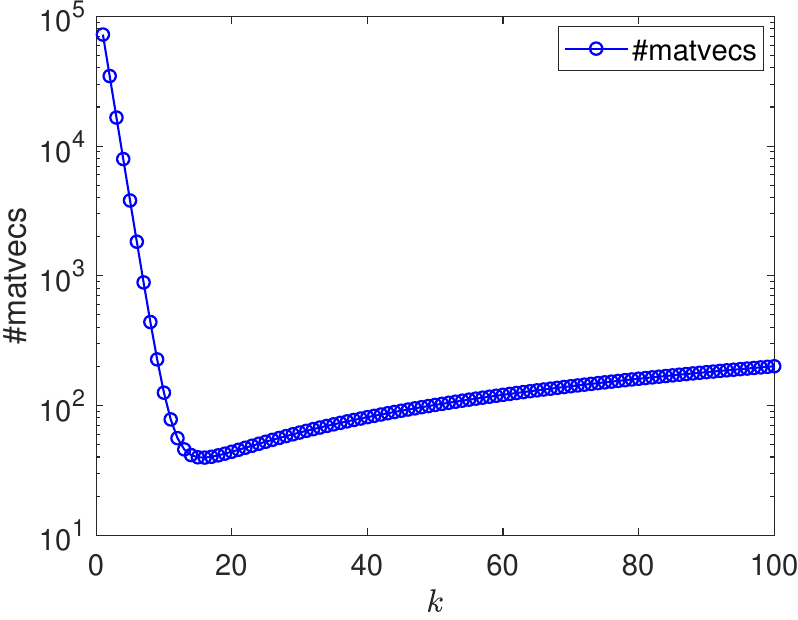}\label{fig:subfig-exp-min}
    \end{minipage}
    }%
    \hspace{1em}
    \subfigure[step]{
    \begin{minipage}[t]{0.45\linewidth}
    \includegraphics[width=1\textwidth]{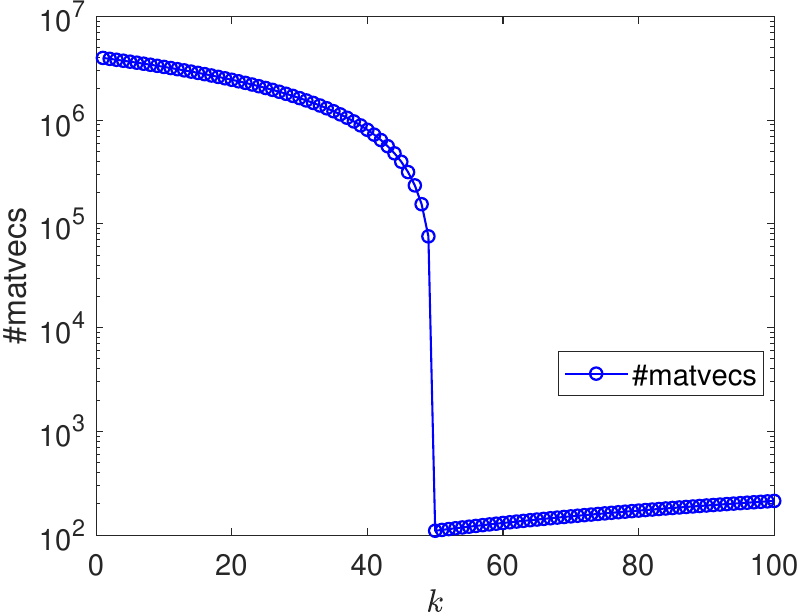}\label{fig:subfig-step-min}
    \end{minipage}
    }%
\centering
\caption{The number of matrix-vector multiplications required for diagonal estimation for matrices with different eigenvalue decay patterns}
\label{fig:min_matvecs}
\end{figure}

Although $\boldsymbol{Q}$ may not achieve the above ideal scenario, the trend of $\#\mathrm{matvecs}$ with respect to the number of columns in $\boldsymbol{Q}$ is consistent with the behavior illustrated in Figure~\ref{fig:min_matvecs}. To select the optimal value of $k$, we refer to the minimum point detection method described in \cite{persson2022improved,chen2023krylov},
\begin{equation}
\label{eq:min_matvecs}
\#\mathrm{matvecs}(k) \geq \#\mathrm{matvecs}(k-1) \geq \#\mathrm{matvecs}(k-2).
\end{equation}
When $\#\mathrm{matvecs}$ increases twice consecutively as $k$ increases, we consider the current $k$ to be near the optimal $k^{*}$ that minimizes $\#\mathrm{matvecs}$. This process of incrementally increasing the number of columns of $\boldsymbol{Q}$ and determining the minimum $\#\mathrm{matvecs}$ is outlined in lines $3$ to $11$ of Algorithm~\ref{algorithm:Adaptive_algorithm}.

\begin{algorithm}
	\caption{Adaptive stochastic diagonal estimation for matrix $\boldsymbol{A}$ (based on RSVD)}
	\label{algorithm:Adaptive_algorithm}
	\begin{algorithmic}[1]
		\STATE{\textbf{Input} Matrix $\boldsymbol{A}\in \mathbb{R}^{n\times n}$, cost function $g$, error tolerance $\varepsilon>0$, failure probability $\delta\in (0,1)$}
		\STATE{\textbf{Initialize} $k=0$, $m_{0}=\infty$, $s =0$, $\mathrm{Temp}\_\mathrm{fro} = 0$, $\mathrm{Temp}\_\mathrm{spe}=0$, $\boldsymbol{Q}=[]$, $\boldsymbol{mole}=\mathrm{zeros}(n,1)$, $\boldsymbol{deno} =\mathrm{zeros}(n,1)$}
		\WHILE{the minimum of $\#\mathrm{matvecs}$ is not detected}
		\STATE{$k=k+1$}
		\STATE{Sample Gaussian random vectors $\boldsymbol{x}_{k}\in \mathbb{R}^{n}$}
		\STATE{$\boldsymbol{y}_{k}=\boldsymbol{A}\boldsymbol{x}_{k}$, $\hat{\boldsymbol{q}}_{k}=(\boldsymbol{I}-\boldsymbol{Q}\boldsymbol{Q}^{\mathrm{T}})\boldsymbol{y}_{k}$}
		\STATE{$\boldsymbol{q}_{k}=\hat{\boldsymbol{q}}_{k}/\|\hat{\boldsymbol{q}}_{k}\|_{2}$}
		\STATE{$\boldsymbol{Q}=[\boldsymbol{Q}\quad \boldsymbol{q}_{k}]$}
		\STATE{$\boldsymbol{d}_{\mathrm{defl}}=\boldsymbol{d}_{\mathrm{defl}}+\mathrm{diag}((\boldsymbol{A}\boldsymbol{q}_{k})\boldsymbol{q}_{k}^{\mathrm{T}})$}
		\STATE{Compute $\#\mathrm{matvecs}$ (by Eq. \eqref{eq:matvecs_function_2})}
		\ENDWHILE
		\WHILE{$m_{s}>s$}
		\STATE{$s=s+1, \alpha_{s}=\mathrm{sup}\left\{\alpha\in (0,1): \frac{\gamma(s/2, \alpha s/2)}{\Gamma(s/2)}\leq \delta\right\}$}
		\STATE{Sample Gaussian random vector $\boldsymbol{\omega}_{s}\in \mathbb{R}^{n}$}
		\STATE{$\boldsymbol{z}_{s} = \boldsymbol{A}((\boldsymbol{I}-\boldsymbol{Q}\boldsymbol{Q}^{\mathrm{T}})\boldsymbol{\omega}_{s})$}
		\STATE{Update $\boldsymbol{mole}=\boldsymbol{mole}+\boldsymbol{\omega}_{s}\odot\boldsymbol{z}_{s}$, and $ \boldsymbol{deno}=\boldsymbol{deno}+\boldsymbol{\omega}_{s}\odot\boldsymbol{\omega}_{s}$}
		\STATE{$\boldsymbol{d}_{\mathrm{rem}} = \boldsymbol{mole}\oslash\boldsymbol{deno}$}
		\STATE{Update $\mathrm{Temp}\_\mathrm{fro}=\mathrm{Temp}\_\mathrm{fro}+\|\boldsymbol{z}_{s}\|_{2}^{2}$, and $\mathrm{Temp}\_\mathrm{spe}=\|\boldsymbol{d}_{\mathrm{rem}}\|_{2}^{2}$}
		\STATE{$m_{s}=g\left(\sqrt{\frac{1}{s\alpha_{s}}\mathrm{Temp}\_\mathrm{fro}-\mathrm{Temp}\_\mathrm{spe}}\right)$}
		\ENDWHILE
		\RETURN{$\boldsymbol{d}_{\mathrm{defl}}+\boldsymbol{d}_{\mathrm{rem}}$}
	\end{algorithmic}
\end{algorithm}

The second challenge in the adaptive diagonal estimation method is computing $\#\mathrm{matvecs}$ (as in line $10$ of Algorithm~\ref{algorithm:Adaptive_algorithm}), particularly the calculation or approximation of $\|\boldsymbol{B}_{\text{off}}\|_{F}^{2}$. This is crucial for identifying the optimal number of columns in $\boldsymbol{Q}$ based on the relationship between $k$ and $\#\mathrm{matvecs}$, as well as determining the number of query vectors $m$ in stochastic diagonal estimation. These two key challenges will be addressed separately in the following sections.

\subsection{Determination of the optimal number of columns in $\boldsymbol{Q}$}

Based on the criterion in Eq. \eqref{eq:min_matvecs}, it is clear that the minimum value of $\#\mathrm{matvecs}$ corresponds to the optimal number of columns in $\boldsymbol{Q}$. In practice, an approximate estimate of $\#\mathrm{matvecs}$ is sufficient for determining the optimal value of $k$ using this criterion. Consequently, we adopt a relatively loose approximation for $\|\boldsymbol{B}_{\text{off}}\|_{F}$ in our calculations. Specifically,
\begin{equation}
	\label{eq:bar_B_Fnorm}
	\begin{array}{ll}
	\left\|\boldsymbol{B}_{\text{off}}\right\|_{F}^{2} &=  \left\|\boldsymbol{B}\right\|_{F}^{2}-\left\|\mathrm{diag}(\boldsymbol{B})\right\|_{2}^{2} \\
	& = \left\|\boldsymbol{A}(\boldsymbol{I}-\boldsymbol{Q}\boldsymbol{Q}^{\mathrm{T}})\right\|_{F}^{2} - \left\|\mathrm{diag}(\boldsymbol{A}(\boldsymbol{I}-\boldsymbol{Q}\boldsymbol{Q}^{\mathrm{T}}))\right\|_{2}^{2} \\
	& = \left\|\boldsymbol{A}\right\|_{F}^{2} -\left\|\boldsymbol{A}\boldsymbol{Q}\right\|_{F}^{2}-\left\|\mathrm{diag}(\boldsymbol{A})-\mathrm{diag}(\boldsymbol{A}\boldsymbol{Q}\boldsymbol{Q}^{\mathrm{T}})\right\|_{2}^{2}.
	\end{array}
\end{equation}
Although $\|\boldsymbol{A}\|_{F}^{2}$ is difficult to calculate directly, it can be considered as a constant since its value is unaffected by changing $k$. Additionally,
$$\left\|\boldsymbol{A}\boldsymbol{Q}\right\|_{F}^{2}=\mathrm{tr}\left((\boldsymbol{A}\boldsymbol{Q})^{\mathrm{T}}(\boldsymbol{A}\boldsymbol{Q})\right)=\sum_{i=1}^{k}\left\|\boldsymbol{A}\boldsymbol{q}_{i}\right\|_{2}^{2},$$ 
where $\boldsymbol{q}_{i}$ denotes the $i$-th column of $\boldsymbol{Q}$, i.e. $[\boldsymbol{Q}]_{:,i}$. Thus, the value of $\|\boldsymbol{A}\boldsymbol{Q}\|_{F}^{2}$ updates with an increasing $k$.

For any matrix $\boldsymbol{X}\in \mathbb{R}^{n\times n}$, $\|\mathscr{D}(\boldsymbol{X})\|_{F}^{2}=\|\mathrm{diag}(\boldsymbol{X})\|_{2}^{2}$. Therefore,
\begin{equation*}
	\left\|\mathrm{diag}(\boldsymbol{A})-\mathrm{diag}(\boldsymbol{A}\boldsymbol{Q}\boldsymbol{Q}^{\mathrm{T}})\right\|_{2}^{2}=\left\|\mathscr{D}(\boldsymbol{A})-\mathscr{D}(\boldsymbol{A}\boldsymbol{Q}\boldsymbol{Q}^{\mathrm{T}})\right\|_{F}^{2}.
\end{equation*}
For an ideal column orthonormal matrix $\boldsymbol{Q}$, $\boldsymbol{A}\succeq \boldsymbol{A}\boldsymbol{Q}\boldsymbol{Q}^{\mathrm{T}}\succeq \mathbf{O}$, implying $\mathscr{D}(\boldsymbol{A})\succeq \mathscr{D}(\boldsymbol{A}\boldsymbol{Q}\boldsymbol{Q}^{\mathrm{T}})\succeq \mathbf{O}$. Based on Lem. 2.3 in \cite{persson2023algorithm},
\begin{equation*}
	\left\|\mathscr{D}(\boldsymbol{A})-\mathscr{D}(\boldsymbol{A}\boldsymbol{Q}\boldsymbol{Q}^{\mathrm{T}})\right\|_{F}^{2} \leq \left\|\mathscr{D}(\boldsymbol{A})\right\|_{F}^{2} - \left\|\mathscr{D}(\boldsymbol{A}\boldsymbol{Q}\boldsymbol{Q}^{\mathrm{T}})\right\|_{F}^{2},
\end{equation*}
we take $\|\mathrm{diag}(\boldsymbol{A})\|_{2}^{2} -\|\mathrm{diag}(\boldsymbol{A}\boldsymbol{Q}\boldsymbol{Q}^{\mathrm{T}})\|_{2}^{2}$ as the approximate upper bound of $\|\mathrm{diag}(\boldsymbol{A})-\mathrm{diag}(\boldsymbol{A}\boldsymbol{Q}\boldsymbol{Q}^{\mathrm{T}})\|_{2}^{2}$ and consider $\|\mathrm{diag}(\boldsymbol{A})\|_{2}^{2}$ as a constant since it is not affected $k$.

In total, we approximate $\#\mathrm{matvecs}$ as
$$2k+g\left(\sqrt{C-\left\|\boldsymbol{A}\boldsymbol{Q}\right\|_{F}^{2}+\left\|\mathrm{diag}(\boldsymbol{A}\boldsymbol{Q}\boldsymbol{Q}^{\mathrm{T}})\right\|_{2}^{2}}\right),$$
where $C$ is a constant\footnote{$C$ can be set to $0$, in which case $2k+g\left(\sqrt{\left\|\boldsymbol{A}\boldsymbol{Q}\right\|_{F}^{2}-\left\|\mathrm{diag}(\boldsymbol{A}\boldsymbol{Q}\boldsymbol{Q}^{T})\right\|_{2}^{2}}\right)$ is maximized.}. Thus, the approximate minimum value of $\#\mathrm{matvecs}$ can be determined without additional computation.

\subsection{Determination of the number of query vector and $\left\|\boldsymbol{B}_{\text{off}}\right\|_{F}$}

 To meet the lower bound of $m$ presented by Eq. \eqref{eq:m_bound_our}, the estimation of $\|\boldsymbol{B}_{\text{off}}\|_{F}$ should be as accurate as possible to minimize computational costs. Based on Eq. \eqref{eq:bar_B_Fnorm},
\begin{equation*}
	\label{eq:determine_m_value}
	\left\|\boldsymbol{B}_{\text{off}}\right\|_{F}^{2}=\left\|\boldsymbol{A}(\boldsymbol{I}-\boldsymbol{Q}\boldsymbol{Q}^{\mathrm{T}})\right\|_{F}^{2} - \left\|\mathrm{diag}(\boldsymbol{A}(\boldsymbol{I} -\boldsymbol{Q}\boldsymbol{Q}^{\mathrm{T}}))\right\|_{2}^{2}.
\end{equation*}
We will discuss these two terms, $\|\cdot\|_{F}^{2}$ and $\|\cdot\|_{2}^{2}$, in the above formula separately.

Firstly, directly computing the norm value of $\|\boldsymbol{A}(\boldsymbol{I}-\boldsymbol{Q}\boldsymbol{Q}^{\mathrm{T}})\|_{F}$ is costly. The following lemma shows that this norm value can be approximated with a few matrix-vector multiplications.

\begin{lemma}[\cite{persson2022improved}, Lem. 2.2]
	\label{lemma:matrix_Fnorm_approx}
	For any matrix $\boldsymbol{B}\in \mathbb{R}^{n\times n}$, let $\boldsymbol{\Omega} \in \mathbb{R}^{n\times k}$ be a standard Gaussian random matrix. For $\alpha \in (0,1)$, we have
	\begin{equation}
		\mathbb{P}\left\{\frac{1}{k}\left\|\boldsymbol{B}\boldsymbol{\Omega}\right\|_{F}^{2}<\alpha \left\|\boldsymbol{B}\right\|_{F}^{2}\right\}\leq \mathbb{P}\left\{X<\alpha\right\}=\frac{\gamma(k/2,\alpha k/2)}{\Gamma(k/2)},
	\end{equation}
where random variable $X$ follows the gamma distribution with shape and scale parameters $k/2$, $\gamma(s,x):= \int_{0}^{x}t^{s-1}\mathrm{e}^{-t}\mathrm{d}t$ is the lower incomplete Gamma function, and $\Gamma(s)$ is the standard Gamma function.
\end{lemma}

According to the \cref{lemma:matrix_Fnorm_approx}, for any given probability $\delta\in (0,1)$, if suitable values of $k$ and $\alpha$ are found such that $\gamma(k/2,\alpha k/2)/\Gamma(k/2)\leq \delta$, then
\begin{equation*}
    \mathbb{P}\left\{\frac{1}{k\alpha}\left\|\boldsymbol{A}(\boldsymbol{I}-\boldsymbol{Q}\boldsymbol{Q}^{\mathrm{T}})\boldsymbol{\Omega}\right\|_{F}^{2}\geq \left\|\boldsymbol{A}(\boldsymbol{I}-\boldsymbol{Q}\boldsymbol{Q}^{\mathrm{T}})\right\|_{F}^{2}\right\}\geq 1-\delta.
\end{equation*}
Thus, the Frobenius norm of $\boldsymbol{A}(\boldsymbol{I}-\boldsymbol{Q}\boldsymbol{Q}^{\mathrm{T}})$ can be estimated by applying a Gaussian random matrix $\boldsymbol{\Omega}$. 

Secondly, since $\mathrm{diag}(\boldsymbol{A}(\boldsymbol{I}-\boldsymbol{Q}\boldsymbol{Q}^{\mathrm{T}}))$ is estimated using a stochastic diagonal estimator, this computation also involves generating random vectors to estimate $\boldsymbol{A}(\boldsymbol{I}-\boldsymbol{Q}\boldsymbol{Q}^{\mathrm{T}})$. To minimize the total number of matrix-vector multiplications, the stochastic diagonal estimation of $\boldsymbol{A}(\boldsymbol{I}-\boldsymbol{Q}\boldsymbol{Q}^{\mathrm{T}})$ can share the same random vectors used in the Frobenius norm estimation. This approach not only ensures the required random operations but also reduces computational overhead. Furthermore, as the number of random vectors increases in the stochastic diagonal estimation process, the value of $(k\alpha)^{-1}\|\boldsymbol{A}(\boldsymbol{I}-\boldsymbol{Q}\boldsymbol{Q}^{\mathrm{T}})\boldsymbol{\Omega}\|_{F}^{2}$ gradually decreases and converges to $\|\boldsymbol{A}(\boldsymbol{I}-\boldsymbol{Q}\boldsymbol{Q}^{\mathrm{T}})\|_{F}^{2}$ \cite{persson2022improved}. At this stage, the estimated number of query vectors $m$ also gradually decreases and eventually converges. Therefore, when estimating both $\|\boldsymbol{A}(\boldsymbol{I}-\boldsymbol{Q}\boldsymbol{Q}^{\mathrm{T}})\|_{F}^{2}$ and $\mathrm{diag}(\boldsymbol{A}(\boldsymbol{I}-\boldsymbol{Q}\boldsymbol{Q}^{\mathrm{T}}))$, utilizing shared matrix-vector multiplications allows the termination condition to be set once the number of random vectors exceeds the theoretical value of the updated query vectors $m$. This strategy minimizes the number of random vectors required in the process.

Finally, for the computation of $\|\boldsymbol{B}_{\text{off}}\|_{F}^{2}$, the term $\|\mathrm{diag}(\boldsymbol{A}(\boldsymbol{I}-\boldsymbol{Q}\boldsymbol{Q}^{\mathrm{T}}))\|_{2}^{2}$ can be derived by using the estimated values of $\mathrm{diag}(\boldsymbol{A}(\boldsymbol{I-\boldsymbol{Q}\boldsymbol{Q}^{\mathrm{T}}}))$ from the stochastic diagonal estimation. The procedures for estimating $m$ and $\mathrm{diag}(\boldsymbol{A}(\boldsymbol{I}-\boldsymbol{Q}\boldsymbol{Q}^{\mathrm{T}}))$ are implemented in lines $12$ to $20$ of Algorithm~\ref{algorithm:Adaptive_algorithm}.

The complete process is outlined in Algorithm~\ref{algorithm:Adaptive_algorithm}. Since Algorithm~\ref{algorithm:Adaptive_algorithm} builds on Algorithm~\ref{algorithm:Prototype_algorithm_based_rsvd} and it incorporates an adaptive selection of the projection subspace size to minimize overall computational cost, it is referred to as the adaptive stochastic diagonal estimation algorithm.

\section{Numerical results}
\label{sec:section4}
 In this section, we test the effectiveness of the proposed adaptive stochastic diagonal estimation method through numerical experiments. The experiments are conducted on three levels. First, we apply the method to matrices with different eigenvalue decay patterns. Second, we compare the performance with several related stochastic diagonal estimation methods. Finally, we evaluate the performance of the proposed adaptive diagonal estimation algorithm through specific examples in practical applications.

 The numerical experiments were conducted in MATLAB R2020a, with all probability bounds set to $\delta=0.01$ and the query vector generated from Gaussian distribution. The accuracy of the diagonal estimation results (relative error) is measured as follows,
 \begin{equation*}
	\frac{\left\|\mathrm{diag}(\boldsymbol{A})-\mathrm{EST}_{\mathrm{diag}(\boldsymbol{A})}\right\|_{2}}{\left\|\mathrm{diag}(\boldsymbol{A})\right\|_{2}},
\end{equation*}
where $\mathrm{EST}_{\mathrm{diag}(\boldsymbol{A})}$ denotes the estimation of $\mathrm{diag}(\boldsymbol{A})$ obtained by different methods.

In the numerical experiments, the symbol $Ada\_diag$ represents the adaptive stochastic diagonal estimation algorithm proposed in this paper.

\subsection{Experimental results on matrices of different eigenvalue decay}
\begin{table}[htbp]
	\centering
	\caption{Computational cost of the adaptive stochastic diagonal estimator under different eigenvalue decays}
	\label{table:adaptive_result_matrix}
        \begin{threeparttable}[b]
	\begin{tabular}{@{}c c c c c c c c c c c c c@{}}
		\toprule
		\multirow{2}{*}{\textbf{p}} & \multicolumn{3}{c}{\textbf{flat}} & \multicolumn{3}{c}{\textbf{poly}} & \multicolumn{3}{c}{\textbf{exp}} & \multicolumn{3}{c}{\textbf{step}} \\
		\cmidrule(lr){2-4} \cmidrule(lr){5-7} \cmidrule(lr){8-10} \cmidrule(lr){11-13}
		& $k$ & $m$ & $\# \mathrm{mv}$\tnote{*} & $k$ & $m$ & $\# \mathrm{mv}$ & $k$ & $m$ & $\# \mathrm{mv}$ & $k$ & $m$ & $\# \mathrm{mv}$  \\
		\midrule
		2 & 3 & 48& 54 & 35& 27 & 97 & 23 & 7 & 53 & 63 & 26 & 152 \\
		3 & 3 & 162 & 168& 50& 34 & 134 & 25 & 7 & 57 & 73 & 45 & 191 \\
		4 & 3 & 636 & 642 & 71 & 42 & 184 & 27 & 8 & 62 & 95 & 76 & 266 \\
		5 & 3 &2614 & 2620 & 100& 56 & 256 & 30 & 7 & 67 & 135 & 153 & 423 \\
		6 & - & - & - & 139 & 77 & 355 & 32 & 7 & 71 & 211 & 329 & 751 \\
		7 & - & - & - & 195 & 106 & 496 & 34 & 8 & 76 & 331 & 893 & 1555\\
		\bottomrule
	\end{tabular}
    \begin{tablenotes}
    \item[*] \footnotesize{In this table, $\# \mathrm{mv}$ denotes $\# \mathrm{matvecs}$.}
    \end{tablenotes}
 \end{threeparttable}
\end{table}

To clearly present the experimental results, we set the sizes of the matrices with four different types of eigenvalue decay patterns ($flat$, $poly$, $exp$ and $step$) described in \Cref{subsection:adaptive_diagonal_experiment} to $5000\times 5000$ (i.e., $n=5000$).
For matrices with $poly$, $exp$, and $step$ eigenvalue decays, the approximation error bound is set to $\varepsilon\|\mathrm{diag}(\boldsymbol{A})\|_{2}$, where $\varepsilon = 2^{-p},\, p= 2,3,\ldots,7$. For matrices with $flat$ eigenvalue decay, the error bound is similarly $\varepsilon\|\mathrm{diag}(A)\|_{2}$, but limited to $\varepsilon = 2^{-p},\, p= 2,3,\ldots,5$, as for $p\geq 6$, the number of required matrix-vector multiplications $\#\mathrm{matvecs}$ exceeds the matrix dimension. In such cases, the diagonal of $\boldsymbol{A}$ should be computed directly using the deterministic standard basis vector $\boldsymbol{e}_{i}^{\mathrm{T}}\boldsymbol{A}\boldsymbol{e}_{i}$.

For each $(\varepsilon,\delta)$ probabilistic error bound, the adaptive diagonal estimation (Algorithm \ref{algorithm:Adaptive_algorithm}) is applied to matrices with varying eigenvalue decays. Each experiment is repeated for $20$ times, and the results are averaged. The results are presented in Table \ref{table:adaptive_result_matrix}.

The results indicate that for $exp$-type matrices, due to the rapid eigenvalue decay and their concentration around zero, only a small number of query vectors are required to meet the specified probabilistic bounds. A similar trend is observed for $poly$-type matrices, where the eigenvalues also decay quickly, with the smallest approaching zero. Hence, diagonal estimation requires few query vectors in these cases, though the size of the projected subspace increases with tighter precision requirements. 

For $step$-type matrices, the results show that the projected subspace size is consistently greater than $50$ (i.e., the point where the eigenvalue gap appears), confirming that the projected subspace dimension for these matrices must exceed a certain threshold for accurate estimation.

Finally, for matrices with $flat$ eigenvalue decay, stochastic diagonal estimation is more efficient when directly performed on the original matrix without projection. In Table \ref{table:adaptive_result_matrix}, the projected subspace size for $flat$-type matrices is $3$, which relates to the method described in Eq. \eqref{eq:min_matvecs} for determining the minimum point of $\#\mathrm{matvecs}$.

\subsection{Comparison with several related diagonal estimation algorithms}
\label{subsection:comparison_Gaussion}

\begin{figure}[htbp]
    \subfigure[flat]{
    \begin{minipage}[t]{0.45\linewidth}
    \includegraphics[width=1\textwidth]{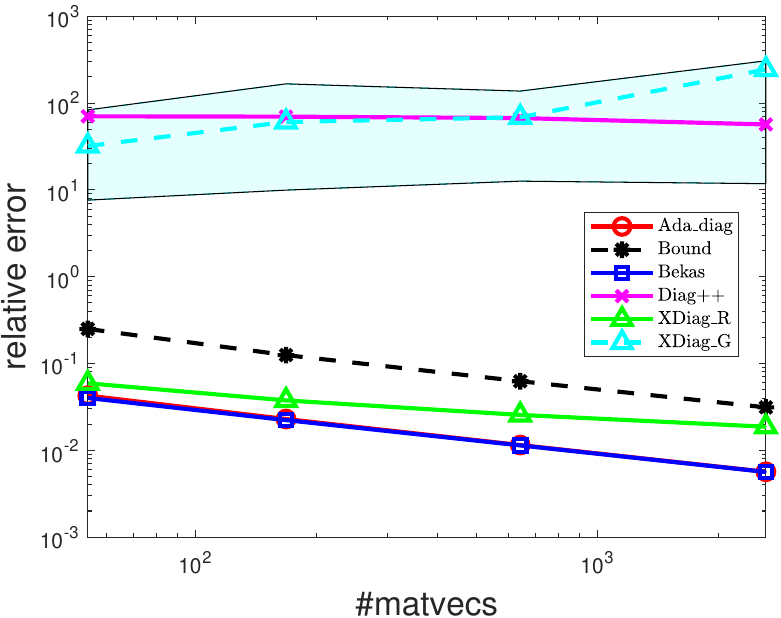}\label{fig:compare_four_flat}
    \end{minipage}%
    }%
    \hspace{1em}
    \subfigure[poly]{
    \begin{minipage}[t]{0.45\linewidth}
    \includegraphics[width=1\textwidth]{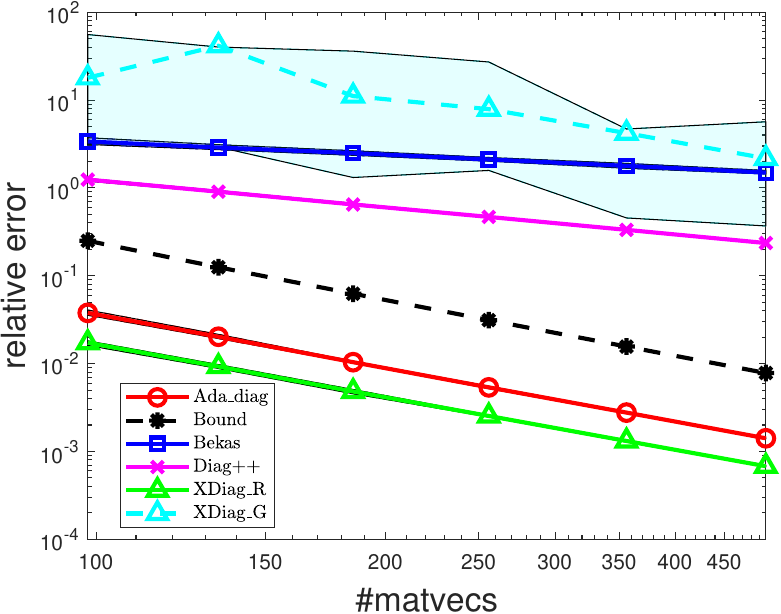}\label{fig:compare_four_poly}
    \end{minipage}%
    }%
    \vspace{1em}
    \subfigure[exp]{
    \begin{minipage}[t]{0.45\linewidth}
    \includegraphics[width=1\textwidth]{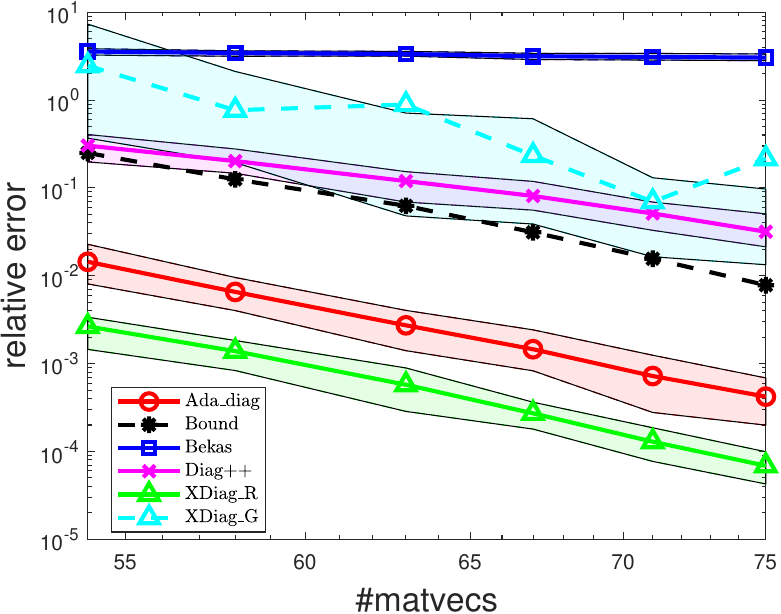}\label{fig:compare_four_exp}
    \end{minipage}
    }%
    \hspace{1em}
    \subfigure[step]{
    \begin{minipage}[t]{0.45\linewidth}
    \includegraphics[width=1\textwidth]{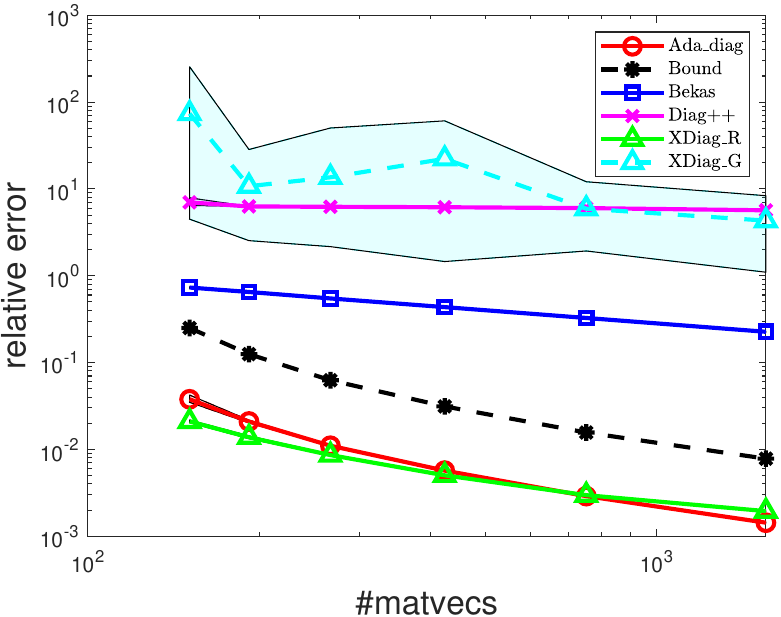}\label{fig:compare_four_step}
    \end{minipage}
    }%
\centering
\caption{Comparison of different diagonal estimation methods under different spectral distributions}
\label{fig:adaptive_compare_four_matrix}
\end{figure}

We compare the adaptive stochastic diagonal estimator proposed in this paper with the stochastic diagonal estimator described by Eq. \eqref{eq:diagonal_estimator}, which was introduced by Bekas et al. \cite{bekas2007estimator} (referred to as Bekas in numerical experiments), Diag++ \cite{baston2022stochastic}, and XDiag \cite{epperly2024xtrace}. The mathematical form of the Diag++ estimation method reads,
\begin{equation}
	\label{eq:diag++_estimator}
	\begin{aligned}
	\mathrm{Diag}++ = & \, \mathrm{diag}(\boldsymbol{Q}\boldsymbol{Q}^{\mathrm{T}}\boldsymbol{A}\boldsymbol{Q}\boldsymbol{Q}^{\mathrm{T}}) +  \\
	& \left[\sum_{i=1}^{\tilde{m}/3}\boldsymbol{\omega}^{(i)}\odot (\boldsymbol{I}-\boldsymbol{Q}\boldsymbol{Q}^{\mathrm{T}})\boldsymbol{A}(\boldsymbol{I}-\boldsymbol{Q}\boldsymbol{Q}^{\mathrm{T}})\boldsymbol{\omega}^{(i)}\right]\oslash \left[\sum_{i=1}^{\tilde{m}/3}\boldsymbol{\omega}^{(i)}\odot \boldsymbol{\omega}^{(i)}\right],
	\end{aligned}
\end{equation}
where $\boldsymbol{Q}=\mathrm{orth}(\boldsymbol{A}\boldsymbol{\Omega})\in \mathbb{R}^{n\times \frac{\tilde{m}}{3}}$ and $\boldsymbol{\Omega}\in \mathbb{R}^{n \times \frac{\tilde{m}}{3}}$ is random matrix. 

The mathematical form of the XDiag estimation method is,
\begin{equation}
	\label{eq:xdiag_estimator}
	\mathrm{XDiag} = \frac{2}{\tilde{m}}\sum_{i=1}^{\tilde{m}/2}\left[\mathrm{diag}\left(\boldsymbol{Q}_{(i)}\boldsymbol{Q}_{(i)}^{\mathrm{T}}\boldsymbol{A}\right)+\frac{\boldsymbol{\omega}^{(i)}\odot \left(\boldsymbol{I}-\boldsymbol{Q}_{(i)}\boldsymbol{Q}_{(i)}^{\mathrm{T}}\right)(\boldsymbol{A}\boldsymbol{\omega}^{(i)})}{\boldsymbol{\omega}^{(i)}\odot \boldsymbol{\omega}^{(i)}}\right],
\end{equation}
where $\boldsymbol{\Omega}=[\boldsymbol{\omega}^{(1)},\boldsymbol{\omega}^{(2)},\ldots,\boldsymbol{\omega}^{(\tilde{m}/2)}]\in \mathbb{R}^{n\times \frac{\tilde{m}}{2}}$ is random matrix,  $\boldsymbol{Q}_{(i)}=\mathrm{orth}(\boldsymbol{A}\boldsymbol{\Omega}_{-i})$ and $\boldsymbol{\Omega}_{-i}$ denote the matrix obtained by removing the $i$-th column of $\boldsymbol{\Omega}$.

In Eqs. \eqref{eq:diag++_estimator} and \eqref{eq:xdiag_estimator}, $\tilde{m}$ represents the total number of matrix-vector multiplications involving matrix $\boldsymbol{A}$ for each estimation method. This differs from $m$, which denotes the number of query vectors in the adaptive stochastic diagonal estimation algorithm. Specifically, $\tilde{m}$ is equivalent to $\#\mathrm{matvecs}$, where $\tilde{m}=2k+m$. For the Bekas algorithm, represented in Eq. \eqref{eq:diagonal_estimator}, $\#\mathrm{matvecs}$, $\tilde{m}$, and $m$ are the same.

The experimental setup for the approximation error $\varepsilon$ follows settings in the previous subsection. For each matrix eigenvalue decay type and $\varepsilon$, we first run the adaptive stochastic diagonal estimation algorithm (Algorithm \ref{algorithm:Adaptive_algorithm}) and record the total matrix-vector multiplications $\#\mathrm{matvecs}$ required to satisfy the probabilistic error bound $(\varepsilon,\delta)$ (corresponding to $\tilde{m}$ in Eqs. \eqref{eq:diag++_estimator} and \eqref{eq:xdiag_estimator}). Using this $\# \mathrm{matvecs}$, we then sequentially execute the Bekas, Diag++, and XDiag methods. Each experiment is repeated $20$ times for each $\varepsilon$, and the average relative error versus matrix-vector multiplications is presented in Figure \ref{fig:adaptive_compare_four_matrix}. Due to the unstable performance of XDiag when using Gaussian random vectors, we also implemented it with Rademacher random vectors, denoted as XDiag\_G and XDiag\_R, respectively. The results show that XDiag\_G performs significantly worse than XDiag\_R.

Figure \ref{fig:adaptive_compare_four_matrix} demonstrates that the adaptive stochastic diagonal estimation method performs similarly to XDiag\_R, with both yielding satisfactory estimation results. In all experiments, the relative errors for these two methods consistently stay within the specified bounds, meeting the required error criteria.

In Figure \ref{fig:compare_four_flat}, for matrices with relatively flat eigenvalues, the Bekas method achieves the best estimation results, while the Diag++ method performs the worst. This occurs because the projection method aims to reduce the variance of the stochastic diagonal estimator and thus reduce the number of query vectors. However, for matrices without significant eigenvalue decay, the projection method does not provide a variance reduction. As shown in Table \ref{table:adaptive_result_matrix}, the adaptive diagonal estimation method captures the eigenvalue variations of the matrix more effectively. When eigenvalues decay slowly, fewer matrix-vector multiplications are used for projection operations, and most of the multiplications are allocated to stochastic diagonal estimation, leading to results comparable to the Bekas method, both outperforming the XDiag\_R method.

Figures \ref{fig:compare_four_poly} and \ref{fig:compare_four_exp} display the estimation results of matrices with $poly$ and $exp$ eigenvalue decay types under different diagonal estimators. Compared to the Bekas method, which relies solely on stochastic diagonal estimation, the Diag++ method, which combines projection with stochastic diagonal estimation, demonstrates a notable advantage for these matrices due to the high eigenvalue decay rate. The Diag++ method captures these larger eigenvalues through orthogonal projection, projecting the matrix into the subspace spanned by the corresponding eigenvectors. This significantly reduces the Frobenius norm of the remaining matrix, thereby decreasing the variance of the diagonal estimator. However, since Diag++ uses a fixed parameter allocation method (i.e., $\tilde{m}/3$ random query vectors and $\tilde{m}/3$ projection subspace size), the overall estimation performance is not optimal. For these matrices, the adaptive diagonal estimation method proposed in this paper and the XDiag\_R method exhibit better estimation performance than the other methods, with XDiag\_R slightly outperforming the adaptive method.

Figure \ref{fig:compare_four_step} shows that the adaptive diagonal estimation method and XDiag\_R deliver comparable performance, both significantly outperform the other methods. As the estimation accuracy increases, the adaptive method gradually surpasses XDiag\_R.

Tables \ref{table:flat} to \ref{table:step} present the numerical results of each set of experiments, averaged over $20$ repetitions.

\begin{table}[htbp]
	\centering
	\caption{For $flat$-type matrix}
	\label{table:flat}
	\begin{tabular}{@{}c c c c c c c@{}}
		\toprule
		\multicolumn{1}{c}{\textbf{err\_bound}} & \multirow{2}{*}{$\# \mathrm{matvecs}$} & \multicolumn{5}{c}{\textbf{relative error}}\\
		\cmidrule(lr){1-1}  \cmidrule(lr){3-7} 
		$\varepsilon$(p) &  & Ada\_diag & Bekas & Diag++ & XDiag\_R & XDiag\_G \\
		\midrule
		0.25(2)    & 54  & 0.0427 & \textbf{0.0401} & 70.4129 & 0.0592 & 32.0407  \\
		0.125(3)   & 168 & 0.0229 & \textbf{0.0223} & 69.7928 & 0.0376 & 60.5825\\
		0.0625(4)  & 642 & 0.0115 & \textbf{0.0114} & 67.2300 & 0.0256 & 68.6676\\
		0.0313(5)  & 2620 & \textbf{0.0056} & \textbf{0.0056} & 56.7330 & 0.0187 & 244.3222 \\
		\bottomrule
	\end{tabular}
\end{table}

\begin{table}[htbp]
	\centering
	\caption{For $poly$-type matrix}
	\label{table:poly}
	\begin{tabular}{@{}c c c c c c c @{}}
		\toprule
		\multicolumn{1}{c}{\textbf{err\_bound}} & \multirow{2}{*}{$\# \mathrm{matvecs}$} & \multicolumn{5}{c}{\textbf{relative error}}\\
		\cmidrule(lr){1-1}  \cmidrule(lr){3-7} 
		$\varepsilon$(p) &  & Ada\_diag & Bekas & Diag++ & XDiag\_R & XDiag\_G\\
		\midrule
		0.25(2)    & 97  & 0.0378 & 3.3765 & 1.2429 & \textbf{0.0173} & 17.9291  \\
		0.125(3)   & 134 & 0.0202 & 2.8811 & 0.9053 & \textbf{0.0093} & 41.5519 \\
		0.0625(4)  & 184 & 0.0104 & 2.4891 & 0.6483 & \textbf{0.0048} & 11.0926 \\
		0.0313(5)  & 256 & 0.0054 & 2.1072 & 0.4674 & \textbf{0.0025} & 7.8725 \\
		0.0156(6)  & 355 & 0.0028 & 1.7809 & 0.3330 & \textbf{0.0013} & 4.2034 \\
		0.0078(7)  & 496 & 0.0014 & 1.4958 & 0.2357 & \textbf{0.0007} & 2.1690 \\
		\bottomrule
	\end{tabular}
\end{table}

\begin{table}[htbp]
	\centering
	\caption{For $exp$-type matrix}
	\label{table:exp}
	\begin{tabular}{@{}c c c c c c c @{}}
		\toprule
		\multicolumn{1}{c}{\textbf{err\_bound}} & \multirow{2}{*}{$\# \mathrm{matvecs}$} & \multicolumn{5}{c}{\textbf{relative error}}\\
		\cmidrule(lr){1-1}  \cmidrule(lr){3-7} 
		$\varepsilon$(p) &  & Ada\_diag & Bekas & Diag++ & XDiag\_R & XDiag\_G \\
		\midrule
		0.25(2)    & 53  & 0.0144 & 3.5700 & 0.3027 & \textbf{0.0026} & 2.4611  \\
		0.125(3)   & 57 & 0.0065 & 3.4384 & 0.2021 & \textbf{0.0014} & 0.7670 \\
		0.0625(4)  & 62 & 0.0027 & 3.3408 & 0.1196 & \textbf{0.0006} & 0.8887 \\
		0.0313(5)  & 67 & 0.0015 & 3.1564 & 0.0810 & \textbf{0.0003} & 0.2338\\
		0.0156(6)  & 71 & 0.0007 & 3.0961 & 0.0512 & \textbf{0.0001} & 0.0696 \\
		0.0078(7)  & 75 & 0.0004 & 3.0402 & 0.0317 & \textbf{6.8688E-5} & 0.2154\\
		\bottomrule
	\end{tabular}
\end{table}
\begin{table}[htbp]
	\centering
	\caption{For $step$-type matrix}
	\label{table:step}
	\begin{tabular}{@{}c c c c c c c@{}}
		\toprule
		\multicolumn{1}{c}{\textbf{err\_bound}} & \multirow{2}{*}{$\# \mathrm{matvecs}$} & \multicolumn{5}{c}{\textbf{relative error}}\\
		\cmidrule(lr){1-1}  \cmidrule(lr){3-7} 
		$\varepsilon$(p) &  & Ada\_diag & Bekas & Diag++ & XDiag\_R & XDiag\_G \\
		\midrule
		0.25(2)    & 152  & 0.0377 & 0.7315 & 7.0001 & \textbf{0.0210} & 73.8531 \\
		0.125(3)   & 191 & 0.0209 & 0.6499 & 6.2827 & \textbf{0.0138} & 10.7539 \\
		0.0625(4)  & 266 & 0.0110 & 0.5476 & 6.2283 & \textbf{0.0086} & 13.6142 \\
		0.0313(5)  & 423 & 0.0057 & 0.4351 & 6.1554 & \textbf{0.0050} & 22.2361 \\
		0.0156(6)  & 751 & \textbf{0.0029} & 0.3262 & 6.0147 & 0.0030 & 5.9212 \\
		0.0078(7)  & 1555 & \textbf{0.0014} & 0.2261 & 5.6738 & 0.0019 & 4.3153\\
		\bottomrule
	\end{tabular}
\end{table}

\subsection{Estimation on practical application problems}

In data mining applications, one common problem is counting the number of triangles in a network graph \cite{avron2010counting}. For an undirected graph with adjacency matrix $\boldsymbol{A}$, the number of triangles connected to vertex $i$ can be computed as $[\boldsymbol{A}^{3}]_{i,i}/2$. Specifically, consider the following two matrices in specific applications \cite{persson2022improved}: one is the Wikipedia vote network matrix \footnote{Data source: \url{https://snap.stanford.edu/data/wiki-Vote.html}}, and the other is the arXiv collaboration network matrix \footnote{Data source: \url{https://snap.stanford.edu/data/ca-GrQc.html}}.

For the experimental settings, the relative error parameter is $\varepsilon=2^{-p}, p=2,3,\cdots,7$. For each value of $\varepsilon$, we first execute the adaptive stochastic diagonal estimation algorithm (Algorithm~\ref{algorithm:Adaptive_algorithm}) to estimate the diagonal and record the required number of matrix-vector multiplications ($\#\mathrm{matvecs}$). Then, using the same number of matrix-vector multiplications, we apply the methods of Bekas, Diag++, and XDiag methods.

\begin{figure}[htbp]
    \subfigure[Wikipedia]{
    \begin{minipage}[t]{0.45\linewidth}
    \includegraphics[width=1\textwidth]{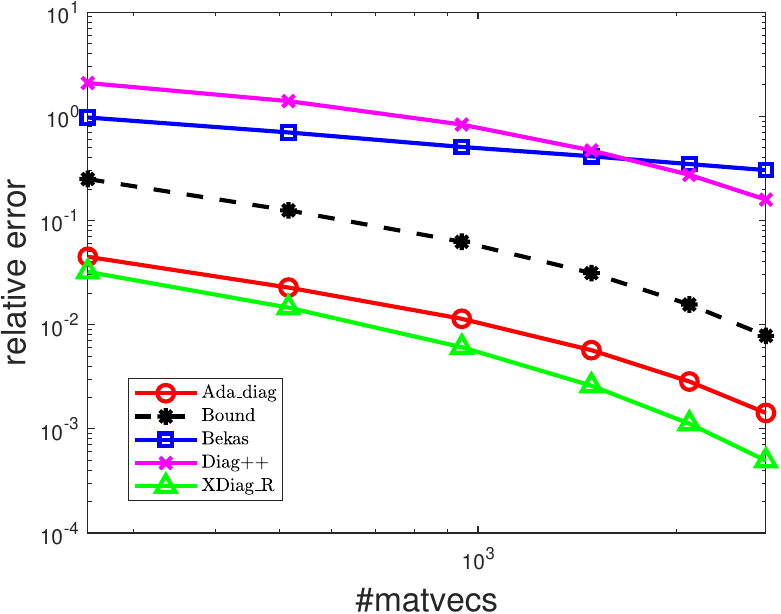}\label{fig:wiki}
    \end{minipage}%
    }%
    \hspace{1em}
    \subfigure[arXiv]{
    \begin{minipage}[t]{0.45\linewidth}
    \includegraphics[width=1\textwidth]{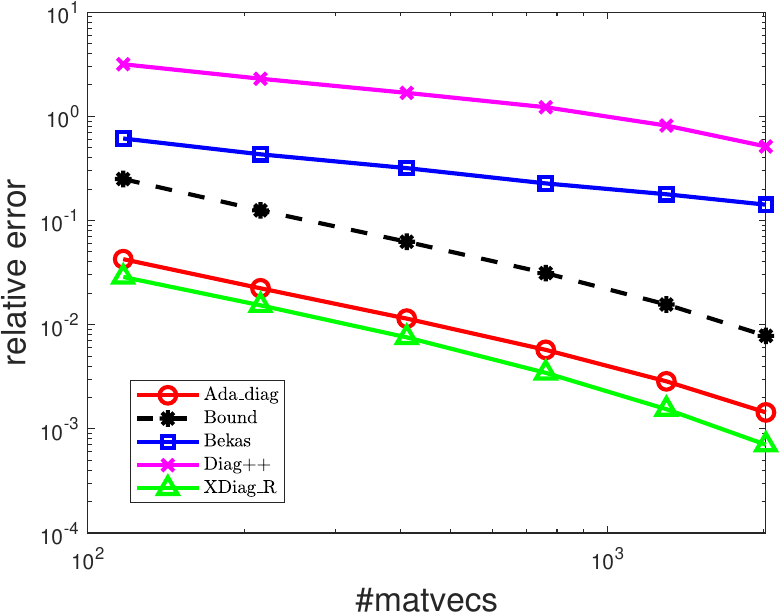}\label{fig:arXiv}
    \end{minipage}%
    }%
\centering
\caption{Comparison of different methods for the number of triangles in undirected graph with respect to Wikipedia and arXiv}
\label{fig:wiki_arXiv}
\end{figure}

\begin{table}[htbp]
	\centering
	\caption{For Wikipedia vote network}
	\label{table:wiki}
	\begin{tabular}{@{}c c c c c c c c @{}}
		\toprule
		\multicolumn{1}{c}{\textbf{err\_bound}} & \multicolumn{3}{c}{ \textbf{assignment} (Ada\_diag)} & \multicolumn{4}{c}{\textbf{relative error}}\\
		\cmidrule(lr){1-1} \cmidrule(lr){2-4} \cmidrule(lr){5-8} 
		$\varepsilon$(p) & $k$ & $m$ & $\# \mathrm{matvecs}$ & Ada\_diag & Bekas & Diag++ & XDiag\_R \\
		\midrule
		0.25(2) & 64 & 124 & 252 & 0.0437& 1.0553 & 1.9638 & \textbf{0.0323} \\
		0.125(3) & 155 & 208 & 518 & 0.0225& 0.6994 & 1.3778 & \textbf{0.0148} \\
		0.0625(4) & 332 & 280 & 944 & 0.0114 & 0.4966 & 0.8355 & \textbf{0.0062} \\
		0.0313(5) & 593 & 306 & 1492 & 0.0059 & 0.4138 & 0.4614 & \textbf{0.0026}\\
		0.0156(6) & 881 & 335 & 2097 & 0.0028 & 0.3400 & 0.2783 & \textbf{0.0011} \\
		0.0078(7) & 1192 & 348 & 2732 & 0.0014 & 0.2934 & 0.1578 & \textbf{0.0005} \\
		\bottomrule
	\end{tabular}
\end{table}

\begin{table}[htbp]
	\centering
	\caption{For arXiv collaboration network}
	\label{table:arXiv}
	\begin{tabular}{@{}c c c c c c c c @{}}
		\toprule
		\multicolumn{1}{c}{\textbf{err\_bound}} & \multicolumn{3}{c}{ \textbf{assignment} (Ada\_diag)} & \multicolumn{4}{c}{\textbf{relative error}}\\
		\cmidrule(lr){1-1} \cmidrule(lr){2-4} \cmidrule(lr){5-8} 
		$\varepsilon$(p) & $k$ & $m$ & $\# \mathrm{matvecs}$ & Ada\_diag & Bekas & Diag++ & XDiag\_R \\
		\midrule
		0.25(2) & 34 & 47 & 115 & 0.0443 & 0.5426 & 3.1312 & \textbf{0.0307} \\
		0.125(3) & 56 & 98 & 210 & 0.0218 & 0.4484 & 2.3472 & \textbf{0.0157} \\
		0.0625(4) & 120 & 169 & 409 & 0.0115 & 0.3094 & 1.7034 & \textbf{0.0077} \\
		0.0313(5) & 225 & 301 & 751 & 0.0058 & 0.2143 & 1.2233 & \textbf{0.0035}\\
		0.0156(6) & 476 & 349 & 1301 & 0.0029 & 0.1673 & 0.8161 & \textbf{0.0015} \\
		0.0078(7) & 791 & 443 & 2025 & 0.0014 & 0.1329 & 0.5120 & \textbf{0.0007} \\
		\bottomrule
	\end{tabular}
\end{table}

Figure~\ref{fig:wiki_arXiv} illustrates the experimental results. The adaptive parameter allocation results for the two example matrices, along with the relative error of the diagonal estimation for various methods, are presented in Table~\ref{table:wiki} and Table~\ref{table:arXiv}. Columns $2$ to $4$ in both tables show the adaptive parameter allocation results of the adaptive diagonal estimation algorithm (Ada\_diag). The XDiag method uses Rademacher random vectors, while all other methods employ Gaussian random vectors.

Overall, the XDiag\_R method shows slightly better estimation performance than the proposed adaptive diagonal estimation method. Both methods meet the specified error-bound requirements and significantly outperform the Bekas and Diag++ methods.

\section{Conclusions}
\label{sec:section5}

To address the challenges of calculating the diagonal of large or implicit matrices, where classical matrix decomposition or directly computing the quadratic form $\boldsymbol{e}_{i}^{\mathrm{T}}\boldsymbol{A}\boldsymbol{e}_{i}$ are either infeasible or computationally expensive, and to overcome the limitations of existing stochastic diagonal estimation methods, such as low efficiency and strong matrix dependency, this paper proposes an adaptive, parameter-optimized stochastic diagonal estimation algorithm. First, we demonstrate through experimental tests and theoretical analysis the effectiveness of projection-based methods in improving the efficiency of stochastic diagonal estimation. We then propose an accurate lower bound for query vectors under given probability and error bounds. Subsequently, we design an adaptive parameter-optimized stochastic diagonal estimation algorithm using approximation techniques combined with a random vector-sharing strategy. Numerical experimental results show that this algorithm is adaptive when handling different matrices. Compared with other stochastic diagonal estimation methods, it significantly improves estimation accuracy and demonstrates stable performance under the same experimental setting.

It is worth noticing that Gaussian random vectors are primarily used in this paper's theoretical analysis and numerical experiments. This choice is based on the fact that in the design of the adaptive algorithm, Gaussian random vectors are used for both matrix Frobenius norm estimation and stochastic diagonal estimation, thereby reducing the computational cost of the algorithm to some extent. However, different types of random vectors have different impacts on estimation accuracy. As shown in the numerical experiments in \Cref{subsection:comparison_Gaussion}, the estimation accuracy of XDiag is significantly better when using Rademacher random vectors than the Gaussian random vectors. Therefore, the analysis and design of adaptive stochastic diagonal estimation algorithms based on Rademacher or other types of random vectors is a direction worth further exploration.


\bibliographystyle{siamplain}
\bibliography{references}
\end{document}